\documentclass{ecai} 

\usepackage{latexsym}
\usepackage{amssymb}
\usepackage{amsmath}
\usepackage{amsfonts}
\usepackage{amsthm}
\usepackage{booktabs}
\usepackage{enumitem}
\usepackage{graphicx}
\usepackage{color}

\usepackage{caption}
\newtheorem{theorem}{Theorem}

\newtheorem{corollary}[theorem]{Corollary}
\newtheorem{proposition}[theorem]{Proposition}

\newtheorem{definition}{Definition}

\newcommand{\BibTeX}{B\kern-.05em{\sc i\kern-.025em b}\kern-.08em\TeX}

\newcommand{\ERF}{\mathsf{ERF}}
\newcommand{\ERL}{\mathsf{ERL}}
\newcommand{\Acts}{\mathsf{A}}
\newcommand{\FER}{\mathsf{FER}}
\newcommand{\Ante}{\mathsf{Ante}}
\newcommand{\Cons}{\mathsf{Cons}}
\newcommand{\CF}{\mathsf{CF}}
\usepackage{listings}
\usepackage{xurl}
\usepackage{hyperref}
\usepackage{breqn}
\usepackage{ragged2e}
\usepackage[htt]{hyphenat}

\usepackage{tikz}
\usetikzlibrary{shapes, positioning, arrows.meta, fit, backgrounds}

\usepackage{algorithm}
\usepackage{algpseudocode}
\usepackage{array}
\makeatletter
\newcommand{\thickhline}{%
    \noalign {\ifnum 0=`}\fi \hrule height 1.5pt
    \futurelet \reserved@a \@xhline
}
\newcolumntype{"}{@{\hskip\tabcolsep\vrule width 1.5pt\hskip\tabcolsep}}
\newcommand{\thickhlineSecond}{%
    \noalign {\ifnum 0=`}\fi \hrule height 1pt
    \futurelet \reserved@a \@xhline
}

\definecolor{codegreen}{rgb}{0,0.6,0}
\definecolor{codegray}{rgb}{0.5,0.5,0.5}
\definecolor{codepurple}{rgb}{0.58,0,0.82}
\definecolor{backcolour}{rgb}{0.95,0.95,0.92}

\lstdefinestyle{mystyle}{
    backgroundcolor=\color{backcolour},   
    commentstyle=\color{codegreen},
    keywordstyle=\color{magenta},
    numberstyle=\tiny\color{codegray},
    stringstyle=\color{codepurple},
    basicstyle=\ttfamily\footnotesize,
    breakatwhitespace=false,         
    breaklines=true,                 
    keepspaces=true,                 
    numbers=left,                    
    numbersep=5pt,                  
    showspaces=false,                
    showstringspaces=false,
    showtabs=false,                  
    tabsize=2,
    xleftmargin=1em,
}

\begin{document}


\begin{frontmatter}


\paperid{123} 


\title{\textit{fEDM+}: A Risk-Based Fuzzy Ethical Decision Making Framework with Principle-Level Explainability and Pluralistic Validation}


\author[A]{\fnms{Abeer}~\snm{Dyoub}\orcid{0000-0003-0329-2419}\thanks{Corresponding Author. Email: abeer.dyoub@uniba.it.}}
\author[A,B]{\fnms{Francesca A.}~\snm{Lisi}\orcid{0000-0001-5414-5844}}

\address[A]{{Dept. of Informatics, University of Bari ``Aldo Moro'', Bari, Italy}}
\address[B]{Centro Interdipartimentale di Logica e Applicazioni (CILA), University of Bari ``Aldo Moro'', Bari, Italy}


\begin{abstract}
In a previous work, we introduced the fuzzy Ethical Decision-Making framework (\textit{fEDM}), a risk-based ethical reasoning architecture grounded in fuzzy logic. The original model combined a fuzzy Ethical Risk Assessment module (\textit{fERA}) with ethical decision rules, enabled formal structural verification through Fuzzy Petri Nets (FPNs), and validated outputs against a single normative referent. Although this approach ensured formal soundness and decision consistency, it did not fully address two critical challenges: principled explainability of decisions and robustness under ethical pluralism.

In this paper, we extend fEDM in two major directions. First, we introduce an Explainability and Traceability Module (ETM) that explicitly links each ethical decision rule to the underlying moral principles and computes a weighted principle-contribution profile for every recommended action. This enables transparent, auditable explanations that expose not only what decision was made but why, and on the basis of which principles. Second, we replace single-referent validation with a pluralistic semantic validation framework that evaluates decisions against multiple stakeholder referents, each encoding distinct principle priorities and risk tolerances. This shift allows principled disagreement to be formally represented rather than suppressed, thus increasing robustness and contextual sensitivity.

The resulting extended fEDM, called \textit{fEDM+}, preserves formal verifiability while achieving enhanced interpretability and stakeholder-aware validation, making it suitable as an oversight and governance layer for ethically sensitive AI systems.
\end{abstract}

\end{frontmatter}


\section{Introduction}
\label{intro}

Artificial agents are increasingly deployed in domains where their actions carry morally significant consequences, including healthcare, assistive robotics, and autonomous decision making. In such contexts, ethical decision-making mechanisms must satisfy two fundamental requirements: (1) structural soundness and formal verifiability, and (2) ethical intelligibility for human stakeholders. AI system that produces normatively acceptable outcomes but cannot justify them in principle-based terms remains difficult to audit and socially contestable. In contrast, a system that is interpretable but structurally flawed risks inconsistency, incompleteness, or hidden conflicts.

In a previous work, we proposed the fuzzy Ethical Decision-Making framework (\textit{fEDM}) \cite{DBLP:journals/corr/abs-2507-01410,DyoubandLisi2026}, a risk-based ethical reasoning framework grounded in fuzzy logic. The model integrates a fuzzy Ethical Risk Assessment module (\textit{fERA}) \cite{DBLP:conf/beware/DyoubL24}, which computes graded ethical risk from ethically relevant factors, with fuzzy ethical decision rules that generate action recommendations. To ensure structural correctness, the rule base is translated into an equivalent Fuzzy Petri Net (FPN) \cite{10.1007/3-540-56863-8_44,liu2023knowledge}, enabling formal verification of incompleteness, inconsistency, circularity, and redundancy. Validation was performed against a single normative referent, representing an agreed-upon ethical stance within the application domain.

Although \textit{fEDM} established formal rigor and operational feasibility, two limitations remained.
First, although decisions were mathematically transparent, their ethical justification remained implicit. The rule base encoded moral considerations, but there was no explicit mechanism to compute how strongly each ethical principle contributed to a given decision. As a result, explanations remained largely rule-centric rather than principle-centric, limiting normative traceability and stakeholder interpretability.
Second, validation against a single referent implicitly assumed ethical consensus. In real-world socio-technical systems, however, stakeholders often differ in their prioritization of principles (e.g., autonomy versus beneficence) and in their tolerance for ethical risk. A single gold standard risks obscuring legitimate pluralism and may incorrectly classify contextually defensible decisions as invalid.

In this work, we introduce \textit{fEDM+} which extends \textit{fEDM} to address these two limitations.
We introduce an \textit{Explainability and Traceability Module} (\textit{ETM}) that augments the existing fuzzy rule base with explicit ethical-principle annotations. Each ethical rule is linked to one or more moral principles, and when rules fire, their activation strengths and certainty factors are aggregated to compute a principle contribution vector for the recommended action. This produces principled explanations that quantify how much each ethical principle influenced the outcome. The \textit{ETM} therefore shifts interpretability from rule transparency to normative transparency.

In addition, we replace single-referent validation with a pluralistic semantic validation framework. Instead of evaluating system outputs against a single benchmark, we define a set of stakeholder referents, each encoding principle priorities, acceptable action ranges, and risk tolerance thresholds. Decisions are assessed relative to these referents, allowing principled disagreement to be formally represented and analyzed. This validation process distinguishes between structural errors (addressed via FPN verification) and normative divergence (captured through referent comparison), thereby increasing robustness and completeness.

Importantly, /these extensions do not alter the core \textit{fERA} mechanism or the FPN-based verification pipeline; rather, they build upon them. The original guarantees of structural correctness are preserved, while normative explainability and validation depth are substantially enhanced.

We demonstrate \textit{fEDM+} using a case study from the healthcare domain, introduced in our previous work. The case study illustrates: (i) principle-annotated rule execution and contribution-vector computation, (ii) generation of traceable ethical explanations, and (iii) comparative validation across multiple stakeholder referents. The results show how the extended model can surface principled trade-offs explicitly and support calibrated alignment adjustments without compromising formal soundness.

The proposed extension positions \textit{fEDM} as a formally grounded yet ethically transparent, and pluralism-aware architecture, suitable for integration as an oversight layer in advanced AI systems operating in morally sensitive domains.

\section{Problem Formulation}
\label{related}

\paragraph{Research Problem:}
\textit{How can a fuzzy-rule-based ethical decision-making framework be formally specified, verified, and semantically validated while also providing transparent traceability from each decision to the ethical principles that justify it?}

This problem requires bridging the technical rigor of fuzzy logic and formal verification with the normative expressiveness of ethical theory. Solving it entails introducing a traceability and explainability layer, to the \textit{fEDM}, that explicitly maps fuzzy decision rules to ethical principles and supports Pluralistic Semantic Validation across diverse ethical referents.

The goal is to produce an EDM model that is not only formally sound and computationally verifiable, but also ethically interpretable, allowing every decision to be explained in terms of its underlying moral rationale.


\subsection{Problem Definition}

We consider the task of \textbf{Ethical Decision Making (EDM)} for artificial agents operating in domains where their actions may have \emph{morally relevant consequences}. The goal is to design a framework that enables such agents to reason about \emph{ethical risk}, make \emph{ethically justifiable decisions}, and provide \emph{transparent explanations} that trace each decision back to the ethical principles it reflects (e.g., autonomy, beneficence, nonmaleficence, and justice according to the biomedical ethics).

Formally, an fEDM problem instance is defined as follows.

\paragraph{Input.}
\begin{enumerate}[label=(\roman*)]
    \item \textbf{Ethically Relevant Factors (ERFs):}
    A finite set of fuzzy input variables
    \[
    E = \{ e_1, e_2, \ldots, e_n \},
    \]
    where each factor \( e_i \) has a domain of crisp values and associated fuzzy sets
    (e.g., Severity = \{LOW, MEDIUM, HIGH\}).

    \item \textbf{Fuzzy Ethical Rules (FERs):}
    Divided into two subsets:
    \begin{itemize}
        \item \textbf{Fuzzy Ethical Risk Rules (FERRs):}
        \[
        (e_i(F_i) \land e_j(F_j) \land \ldots) \Rightarrow Risk(F_r), \ \beta
        \]
        mapping ERFs to fuzzy risk levels, with certainty factor \(\beta \in [0,1]\).

        \item \textbf{Fuzzy Ethical Decision Rules (FERDs):}
        \[
        (Risk(F_r) \land e_i(F_i) \land \ldots) \Rightarrow Action(F_a), \ \beta
        \]
        mapping risk levels and ERFs to fuzzy actions or decisions.
    \end{itemize}

    \item \textbf{Principle Mapping Function (T):}
    A traceability mapping
    \[
    T : FER \to P
    \]
    associating each fuzzy rule with one or more ethical principles. The set of all possible ethical principles is domain dependent.
    This mapping supports explainability by making each decision's ethical rationale explicit.
\end{enumerate}

\paragraph{Output.}
A recommended \textbf{action or decision} \(a \in A\) chosen from a finite action set, together with:
\begin{itemize}
    \item its associated \emph{fuzzy ethical risk level}, and
    \item its \emph{principle trace}, i.e., the set of ethical principles most strongly activated during decision inference.
\end{itemize}

\subsection{Correctness Criteria.}
An fEDM model is considered correct if it satisfies the following properties:
\begin{enumerate}[label=(\roman*)]
    \item \textbf{Structural Soundness:}
    The fuzzy rule base, when mapped into a Fuzzy Petri Net (FPN), is free of structural errors such as
    incompleteness, inconsistency, circularity, and redundancy.

    \item \textbf{Pluralistic Semantic Validity:}
    For any given input, the model’s outputs are consistent with at least one expert-defined referent
    from a set of stakeholder perspectives
    \[
    \mathcal{V} = \{ V_1, V_2, \ldots, V_m \},
    \]
    capturing ethical pluralism and context-dependent acceptability.

    \item \textbf{Traceability Integrity:}
    Each decision must be explainable by the ethical principles \(T(a)\) that justify it,
    enabling human auditors to verify both procedural and normative adequacy.
\end{enumerate}

\subsection{Formal Problem Statement}

Given:
\begin{itemize}
    \item a set of ethically relevant factors \(E\),
    \item a fuzzy rule base \(R = \{FERRs, FERDs\}\),
    \item and a traceability mapping \(T : FERD \to P\),
\end{itemize}
the goal is to design a fuzzy Ethical Decision-Making model \textit{fEDM}
such that:
\begin{enumerate}[label=(\alph*)]
    \item its rule base can be formally verified via translation into a Fuzzy Petri Net (FPN);
    \item its decisions are validated through pluralistic semantic validation against expert referents; and
    \item each decision is traceable to explicit ethical principles.
\end{enumerate}

\noindent
The central challenge is to ensure that EDM models are not only executable and verifiable,
but also \emph{ethically transparent and pluralistically valid},
thereby enabling trustworthy and explainable moral reasoning in artificial agents.

\section{The Extended EDM Model }
\label{proposed}


In our proposed model, EDM is mainly based on the level of ethical risk calculated by a specific module fERA of the fEDM system.
fERA is also based on fuzzy logic (see Section \ref{era}).
Various input and output fuzzy variables that dictate the required ethical behaviour of the machine are identified. Then, it is expressed in terms of fuzzy rules for EDM. These rules map the inputs, and risk levels to outputs (actions/decisions) of the system (FERDs (see Definition 1)).

\subsection{Fuzzy Rule-Based EDM Model}
\label{fuzzyEDM}

\begin{figure}[htbp!]
  \centering
  \includegraphics[width=0.8\linewidth]{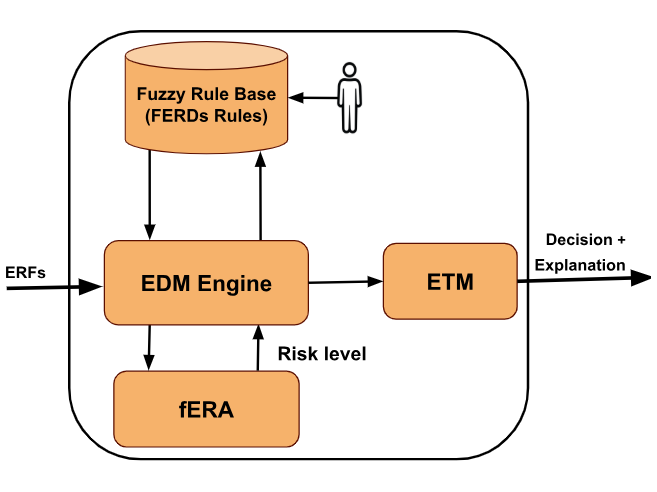}
  \caption{General Structure of \textit{fEDM+}}
  \label{fig:EDM}
\end{figure}

 \vspace{15pt}

Figure \ref{fig:EDM} shows the main components of \textit{fEDM+} which are described in the following:
\begin{itemize}
    \item The fERA module receives from the EDM engine information about the case at hand, particularly the relevant factors/parameters for ethical risk assessment. fERA uses this information to  calculate the risk level value.
    \item The EDM Engine, at minimum, performs the following functions:
    \begin{itemize}
        \item accepts input about the ethically relevant factors (ERFs) of the case at hand; passes them to the fERA module and receives from fERA the risk level.
        \item consults the Fuzzy Rule Base (FERDs) to generate its responses (actions/decisions).
        \item can learn fuzzy decision rules from data (if available), and add them to the Fuzzy Rule Base to be used later for decision making (in the current implemented version, fuzzy inference rules are written manually by human experts).
    \end{itemize}
    \item The Fuzzy Rule Base contains a set of fuzzy rules that define the behavior of the system. These rules describe how fuzzy inputs and fuzzy risk values relate to the outputs (actions/decisions) based on human expert knowledge.
    \item The ETM module: each rule is annotated with one or more ethical principles. When a decision is made, the ETM records which rules were fired, then constructs a justification by mapping the dominant rules and their outputs to these principles.
\end{itemize}

The output of the \textit{fEDM+} becomes:
\begin{itemize}
    \item Action (as before).
    \item principle contribution vector (traceability).
    \item Human-readable explanation.
\end{itemize}
To ensure the correctness and validity of these EDMs, it is essential to verify and validate them during the development process. However, in order to have a reliable EDM model on which many ethical decisions depend, it is fundamental to ensure that the EDM model passes the V\&V criteria.

Now, we give the following formal definition of \textit{fEDM+} that incorporates explainability and traceability.



\begin{definition}[Extended fEDM Model (\textit{fEDM+})]
The $fEDM+$ model is an 7-tuple:

\[
\text{fEDM+} = \left( M, \ERF, \ERL, \Acts, \FER, P, T \right)
\]

where:

\begin{itemize}
    \item $M$ is the model name.
    
    \item $\ERF = \{ e_1, \dots, e_n \}$ is a finite set of fuzzy inputs that represent ethically relevant factors (crisp inputs).
    
    \item $\ERL = \{ r_1, \dots, r_m \}$ is the set of ethical risk levels (fuzzy internal states).
    
    \item $\Acts = \{ a_1, \dots, a_k \}$ is a finite set of possible actions/decisions.
    
    \item $\FER = \{FERRs \cup FERDs\}$ is the rule base, a finite set of fuzzy ethical rules:
    \begin{itemize}
        \item $FERRs$: fuzzy ethical rules for risk assessment.
        \item $FERDs$: fuzzy ethical rules for decision making.
    \end{itemize}
    
    
    \item $P = \{P_1, \dots, P_n\}$ is the universe of ethical principles the system must cover depending on the domain (e.g. in the medical field, $P_1 = \text{Autonomy}$, $P_2 = \text{Non-maleficence}$, \dots).
    
    \item $T : \FER \to 2^{P}$ is the traceability map sending each rule $R$ to its tagged principle(s), i.e., $T(R) = P(R) \subseteq P$.
    
\end{itemize}
\end{definition}

\begin{definition}[Fuzzy Ethical Rule (Extended)]
Each rule $R \in \FER$ is now a 5-tuple:
    \[
 R = \left( \text{Rname}, \Ante, \Cons, \CF, P(R) \right)
\]

where 
\begin{itemize}
\item $Rname$, is the name of the rule.
    \item $Ante$ is the set of antecedents of the rule. An $AntS$ is a disjunction of conjunctions of antecedents. We denote an antecedent by $A(F)$, where $A$ could be an input (ERF) or an internal state (ERL), and $F$ is a fuzzy set. 
    \item $ConS$ is the set of consequents of the rule. We denote a consequent by $C(F)$, where $C$ could be an internal state (ERL), or an output action/decision, and $F$ is a fuzzy set. 
    \item $\CF \in [0,1]$ is the rule certainty factor.
    \item $P(R) \subseteq P$ is the non-empty set of ethical principles that $R$ instantiates (e.g., $\{\text{Autonomy}, \text{Beneficence}\}$).
\end{itemize}
\end{definition}

According to the above definition of the $fEDM$ model, we can distinguish between three types of rules: Type1) $ERFs \xrightarrow{}RLs$, Type2) $RLs \xrightarrow{} AS/DS$, Type3) $ERFs \land RLs \xrightarrow{} AS/DS$.

\subsection{The fERA module}
\label{era}
When developing ethical AI systems, the focus should be on identifying the ethical risks associated with both the system and its use. The primary concern is recognizing these risks, rather than delving into the ethical theories that justify why something is considered ethical. Key questions include: What ethical risks are present in the system we are building? How might users employ the system or product in ways that pose ethical risks (deployment considerations)? In response to this, the development team must carefully consider which features to include or exclude in the AI system to effectively mitigate these risks \cite{blackman2022ethical}.
As symbiosis in SAI systems increases, so does the potential for ethical risks, making effective ERA essential. However, due to the inherent fuzziness and incomplete understanding of complex systems, managing this uncertainty is crucial for ensuring ethical behavior and minimizing harm.
We propose a fuzzy logic based system for ERA. The main components of this system for are:
\begin{description}
    \item[Inputs] These are the factors/parameters relevant for the ethical risk calculation.
	\item[Fuzzification] In this stage crisp input values are converted into fuzzy sets, , allowing real-world data (e.g., temperature, speed) to be interpreted in a way that accounts for uncertainty or vagueness. This is done using membership functions that map input values to a degree of membership between 0 and 1. 
	\item[Inference Engine] The inference engine will consults the  \textit{Fuzzy Rules Base} that contains FERRs rules which are a set of "if-then" rules that define the system's behavior. These rules describe how fuzzy inputs relate to the fuzzy output (the ethical risk) based on expert knowledge. The engine will apply these rules to the fuzzified input to derive fuzzy output sets. It determines which rules are relevant based on the degree of membership of the input values.
	\item[Defuzzification] Converting the fuzzy output sets back into crisp values to implement actions or decisions. Common defuzzification methods include \textit{centroid}, \textit{mean of maximum}, and \textit{bisector}, etc.\footnote{\url{https://it.mathworks.com/help/fuzzy/defuzzification-methods.html}} 
 \item[Output] The only Output in our fuzzy system is the ethical risk level. The computed level can be subsequently considered to make the appropriate decision/action to mitigate that risk.
\end{description}

Figure \ref{fig:fERA} shows the architecture of the fERA module whose approach will be illustrated by means of a case study in Section \ref{case}. For more details of fERA, see \cite{DBLP:conf/beware/DyoubL24}.

\subsection{The ETM Module}
\label{etm}
\paragraph{Ethical Principle Annotation}

Let $P = \{P_1, P_2, ..., P_n \}$ denote the set of core domain ethical principles. Each fuzzy ethical rule $r \in R$ is annotated with a subset of principles $P_r \subseteq P$, representing the ethical rationale underlying that rule.

For example:

\begin{quote}
\textbf{$r$:} \textbf{If} Severity = HIGH \textbf{and} Mental = BAD \textbf{then} Risk = HIGH (CF = 0.9)
\end{quote}

may be annotated with $P_r = \{\text{Nonmaleficence}, \text{Beneficence}\}$.

\paragraph{Rule Traceback}
When the fEDM engine selects an action $a \in A$, it does so by activating a set of rules $\text{\textit{Fired}}(a) \subseteq R$. For each rule $r \in \text{\textit{Fired}}(a)$, the ETM records:

\begin{itemize}
    \item the degree of activation $\mu_r \in [0,1]$,
    \item the certainty factor $\CF \in [0,1]$, and
    \item the annotated principle set $P_r$.
\end{itemize}

\paragraph{Principle Mapping}
The contribution of a principle $p \in P$ to decision $a$ is computed as:

\[
\text{Score}(p,a) = \sum_{\substack{r \in \text{Fired}(a), \\ p \in P_r}} \mu_r \cdot \CF_r
\]

$\mu_r \cdot \CF_r$ is the effective firing strength of a fuzzy rule $r$ \cite{BELYADI2021381}.
The explanation trace for action $a$ is then given by:

\[
\text{Expl}(a) = \left\{ \left(p, \text{Score}(p,a)\right) \mid p \in P \right\}
\]

This yields a weighted distribution of ethical principles justifying the decision. The dominant principle(s) correspond to those with the highest contribution scores.

\begin{figure}[htb!]
  \centering\includegraphics[width=0.9\linewidth]{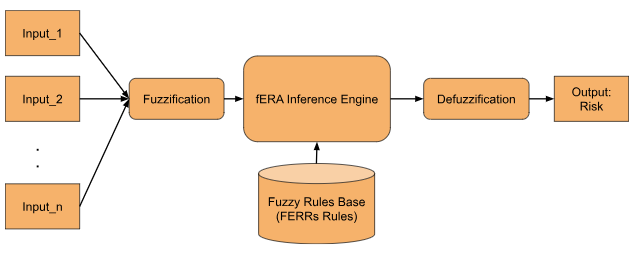}
  \caption{Architecture of the fERA Module}
  \label{fig:fERA}
  \vspace{1cm}
\end{figure}
\subsection{Pluralistic Validation Referents}
\label{vreferent}

Let
\[
\textbf{V} = \{ V_1, V_2, \ldots, V_k \}
\]
be a set of \emph{stakeholder referents}, each representing a distinct ethical or evaluative perspective 
(e.g., human operator, policy authority, user group, etc.).
Each referent $V_i$ captures a normative viewpoint about what constitutes ethically acceptable behavior 
for a decision-making agent in a given application context.

Each referent $V_i$ is formally defined as a tuple:
\[
V_i = ( M_i, \; \Pi_i, \; \rho_i, \; \Theta_i, \; \tau_i )
\]
where:
\begin{itemize}
    \item \( M_i: E^* \to \mathcal{F}(A) \) is the referent's expected mapping from ethically relevant factors \( x \in E^* \) to a fuzzy distribution over actions \( A \) (i.e., a function that returns \( \mu_{M_i(x)}(a) \in [0,1] \) for each action \( a \)). \( M_i \) encodes what the stakeholder expects the action distribution should look like for a given case.
    \item \( \Pi_i \) is a principle priority ordering (partial or total) over the set of principles \( P = \{ p_1, \dots, p_m \} \). This ordering is qualitative and used for interpretation/explanation (not for driving decisions).
    \item \( \rho_i \in [0,1] \) is the referent's subjective risk tolerance (higher \( \rho_i \) = more tolerant of risk).
    \item \( \Theta_i: [0,1] \to 2^A \) maps an estimated risk level \( r \) to the set of actions the referent considers acceptable at that risk.
    \item \( \tau_i \in [0,1] \) is the referent's semantic tolerance (smaller \( \tau_i \) = stricter acceptance).
\end{itemize}

For a given input vector $x \in E^*$, the $fEDM+$ model produces:
\begin{itemize}
    \item \( f_{\text{model}}(x) \in \mathcal{F}(A) \) --- model's fuzzy action distribution (\( \mu_f(a) \)),
    \item \( r(x) \in [0,1] \) --- inferred crisp/normalized risk,
    \item \( p_{\text{model}}(x) \in [0,1]^m \) --- principle-score vector from ETM (one score per principle).
\end{itemize}

\paragraph{Evaluation Metrics}
\begin{enumerate}
    \item \textbf{Action Similarity
} Each referent \( V_i \) compares model action to its own expected mapping \cite{PAPAKOSTAS20131609,mordeson2024fuzzy}. The formula for computing the action similarity score for input \(\mathbf{x}\), normalized in \([0, 1]\) (1 = perfect similarity, 0 = no similarity), is as follows:

\[
S_A^{(i)}(\mathbf{x}) = \frac{\sum_{a \in A} \mu_S(a) \cdot \mu_R(a)}{\max\left(\max_{a \in A} \mu_S(a), \max_{a \in A} \mu_R(a)\right)}
\]

Where:
\begin{itemize}
\item \(\mathbf{x}\): input vector values.
\item \(A\): set of discrete actions.
\item \(a\): specific action in \(A\).
\item \(\mu_S(a)\): membership degree of action \(a\) in the system's decision distribution.
\item \(\mu_R(a)\): membership degree of action \(a\) in the referent's expected action distribution (typically 1 for expected action, 0 otherwise).
\item \(\sum_{a \in A}\): summation over all actions \(a \in A\).
\item \(\max_{a \in A} \mu_S(a)\): maximum membership degree in the system's distribution.
\item \(\max_{a \in A} \mu_R(a)\): maximum membership degree in the referent's distribution.
\item \(\max(\cdot, \cdot)\): maximum between the two distribution maxima.
\end{itemize}


\item \textbf{Principle-Order Consistency (Ordinal Test)}
Let \( O_i = \{ (p_u, p_v) \mid p_u \succ_i p_v \} \) be the set of ordered principle pairs from \( \Pi_i \). For small tolerance \( \epsilon \geq 0 \) (e.g., 0.02 to tolerate numerical noise), define
\[
\operatorname{sat}_i(p_u, p_v, x) = 
\begin{cases}
1 & \text{if } p_{\text{model}}(x, p_u) \; \geq \; p_{\text{model}}(x, p_v) - \epsilon, \\
0 & \text{otherwise}.
\end{cases}
\]
is a binary satisfaction indicator that checks whether the model respects one stakeholder’s expected ordering between two ethical principles for a given input $x$. Thus, if the model assigns principle $p_u$ a score at least as high as 
$p_v$ (within a small tolerance $\epsilon$), 
the ordering is respected and $\text{sat}_i = 1$.
Otherwise, the ordering is violated and $\text{sat}_i = 0$.
Then
\[
S_P^{(i)}(x) \; = \; \frac{1}{|O_i|} \sum_{(p_u, p_v) \in O_i} \operatorname{sat}_i(p_u, p_v, x).
\]
\( S_P^{(i)} \in [0,1] \) measures how many ordering constraints the model's principle scores satisfy.
\end{enumerate}

\paragraph{Validation Rule}
The $fEDM+$ model is said to be \emph{semantically valid} if, for every input vector $x \in E^*$ ($E^* = ERF$),
there exists at least one referent \(V_i\) such that:
\[
S_A^{(i)}(x) \; \geq \; 1 - \tau_i \quad \text{and} \quad S_P^{(i)}(x) \; \geq \; 1 - \tau_i,
\]
(i.e., both action-level and principle-order agreement are within the referent's tolerance).
Instead of comparing model outputs to a single referent, we validate against each $V_i \in V$, yielding a vector of semantic-validation scores that reflect different stakeholder perspectives.
This criterion generalizes classical single-referent validation by allowing model outputs to be considered acceptable when they align with the ethical expectations of \emph{at least one} legitimate stakeholder, thus formally incorporating ethical pluralism into the validation process. 

\section{Formal Verification and Validation}

Modeling EDM as a fuzzy rule-base, we may encounter the \textit{structural errors} that typically affect any fuzzy rule-base, among which the most popular are the following \cite{1185845}: 

\begin{itemize}
    \item \textit{Incompleteness}: It results from missing rules.
    \item \textit{Inconsistency}: Inconsistent rules lead to conflicts and should be removed from the rule base. This occurs when a set of rules produces contradictory conclusions under certain conditions. 
    \item \textit{Circularity}: Circular rules occur when several rules have a circular dependency. This circularity can lead to an infinite reasoning loop and must be resolved. 
    \item \textit{Redundancy}: Redundant rules are unnecessary in a rule base. They increase the size of the rule base and may lead to additional, useless deductions. 
\end{itemize}

\subsection{Verification Criteria of EDM Models}
\label{ver}

\begin{proposition}[Fuzzy Rule-Base to FPN Equivalence]
\label{prop:rule_to_fpn}
For any fuzzy Ethical Decision-Making model 
$\text{fEDM+} = \left( M, \ERF, \ERL, \Acts, \FER, P, T \right)$,
where 
$FER = \{\text{FERRs}, \text{FERDs}\}$ 
is a finite fuzzy rule base, there exists an equivalent Fuzzy Petri Net (FPN) 
$N$
such that:
\begin{enumerate}[label=(\roman*)]
    \item each antecedent and consequent in $FER$ is represented by a place in $N$;
    \item each rule in $FER$ is represented by a transition in $N$;
    \item certainty factors and principals annotations are preserved as transition weights; and
    \item the execution semantics of $N$ correspond to the inference semantics of $FER$.
\end{enumerate}
\end{proposition}

\begin{proof}[Proof Sketch]
Let $FER$ contain $m$ fuzzy rules of the form:
\[
r_i : 
(e_1(F_1) \wedge \cdots \wedge e_k(F_k)) 
\Rightarrow c(F_c), \, \beta_i,
\]
where each antecedent $e_j(F_j)$ and consequent $c(F_c)$ 
is a fuzzy proposition associated with certainty factor $\beta_i \in [0,1]$.

\textbf{Place Mapping.}
For each fuzzy proposition 
$p \in \{ e_1(F_1), \ldots, e_k(F_k), c(F_c) \}$,
define a corresponding place $p_j \in P$.

\textbf{Transition Mapping.}
For each rule $r_i$, define a transition $t_i \in T$.

\textbf{Arc Construction.}
For each antecedent $p_j$, add an input arc $(p_j, t_i) \in I$.
For each consequent $q$, add an output arc $(t_i, q) \in O$.

\textbf{Weight Assignment.}
Assign the certainty factor $\beta_i$ of rule $r_i$, and the set of annotating principles $P(r_i)$ 
as the weight of transition $t_i$.

\textbf{Semantics Preservation.}
Firing $t_i$ in $N$ corresponds to applying fuzzy inference for $r_i$:
the marking update propagates truth degrees from antecedent places 
to consequent places according to $\beta_i$.

Thus, the FPN
\[
N = (P, T, I, O, f, \alpha, \gamma)
\]
constructed in this manner is behaviorally equivalent to 
the fuzzy rule base $FER$. \qedhere
\end{proof}

\begin{corollary}[Verifiability]
Since every \textit{fEDM+} model can be mapped to an equivalent FPN, the
structural verification of \textit{fEDM+} models reduces to the structural 
verification of their corresponding FPNs. 
Therefore, existing Petri net–based techniques for detecting 
\emph{incompleteness}, \emph{inconsistency}, \emph{circularity}, 
and \emph{redundancy}
can be directly applied to \textit{fEDM+} models.
\end{corollary}

The verification process aims to detect and eliminate structural errors in the rule base. It consists of four stages:
\begin{enumerate}
    \item Normalization of rules: We apply the procedure introduced in \cite{1185845} to put any rule in the following form:    
    $P_1 \wedge P_2 \wedge ... P_{j-1} \rightarrow P_j$
    \item Transformation of the rule-base to a fuzzy Petri net (FPN): The antecedents, the consequents and the name of a rule are mapped into input places, output places, and transition of the FPN, respectively.
    \item Generation of the FPN reachability graph: Here, we adopt the algorithm presented in \cite{814327} to generate a  graph where each node is a marking of the FPN, and each directed edge represents transition firing and connects one node to the other.
    \item Error detection: We verify the incompleteness, inconsistency, circularity, and redundancy of the rule base by checking the FPN reachability graph \cite{814327}.
\end{enumerate}

Adding the $ETM$ module to our \textit{fEDM} implies the following further checks:
\begin{itemize}
 \item Principle‐Coverage Checks: Beyond checking for dangling antecedents or circular rules, we now also verify that every high‐level ethical principle (e.g. Autonomy, Beneficence, Non-maleficence) is covered by at least one rule in the combined FERR + FERD net. If a principle has no associated rules, that flags an incompleteness error in the ethical specification.

\item Consistency Across Principles: We can detect “cross-principle” inconsistencies: for example, if two rules tagged with mutually exclusive principles (say, Autonomy vs. Non-maleficence) can fire simultaneously under the same inputs, the reachability graph will expose conflicting transitions.

\item Redundancy Among Principle Rules: Rules that duplicate the same principle tag and antecedent–consequent mapping become easier to spot, letting us prune redundant ethical regulations.
\end{itemize}


\subsubsection{Principle-Level Structural and Ethical Verification}
\label{subsec:principle_verification}

Now, we formalize principle-level verification within the fEDM-FPN framework so that it can be stated, tested, and proven as part of the structural verification pipeline. 

In the \textit{fEDM+} model, each fuzzy rule 
$r_i \in R = \{\text{FERR}, \text{FERD}\}$ 
is explicitly associated with one or more ethical principles 
through a \emph{traceability function}:
\[
T : R \rightarrow 2^{P},
\]
Each rule can thus be expressed as:
\[
r_i : (\text{Antecedent}_i) \Rightarrow (\text{Consequent}_i), \; \beta_i, \;
T(r_i) = P_i \subseteq P.
\]
In the corresponding FPN representation 
$N = (P, T, I, O, f, \alpha, \gamma)$,
transitions $t_i \in T$ correspond to rules $r_i$,
and each transition inherits its ethical principle labels via $T(r_i)$.

To ensure that the ethical rule base is both \emph{structurally sound} and 
\emph{ethically coherent}, we define three complementary verification checks:
\begin{enumerate}[label=(\roman*)]
    \item \textbf{Principle Coverage},
    \item \textbf{Cross-Principle Consistency}, and
    \item \textbf{Principle Redundancy}.
\end{enumerate}

\paragraph{Principle Coverage Verification (Completeness of Principle Tagging): }
\label{subsubsec:principle_coverage}

The objective of coverage verification is to confirm that each ethical principle 
$p \in P$ is represented by at least one rule in $R$.
Formally, we can define the coverage function:
\[
\text{Coverage}(p) = 
\begin{cases}
1 & \text{if } \exists r_i \in R : p \in T(r_i), \\
0 & \text{otherwise.}
\end{cases}
\]
The coverage vector is therefore:
\[
C = [\text{Coverage}(p_1), \ldots, \text{Coverage}(p_n)].
\]
The model satisfies the \emph{Principle-Coverage completeness} property if and only if:
\[
\forall p \in P, \; \text{Coverage}(p_i) = 1.
\]
Missing coverage, i.e.,
$\exists p_j \in P : \text{Coverage}(p_j) = 0$,
indicates an \emph{incompleteness error} in the rule base (flags incompleteness error in the ethical specification).
In the FPN, this check corresponds to verifying that for every principle $p$, 
at least one transition $t_i$ is labeled with $p$.

\paragraph{Cross-Principle Consistency Verification: }
\label{subsubsec:cross_principle_consistency}

To detect \emph{cross-principle inconsistencies}, we define a set 
$\mathcal{C} \subseteq P \times P$ of \emph{incompatible principle pairs}, 
such as $(\text{Autonomy}, \text{Nonmaleficence})$.
A conflict arises when two rules $r_i, r_j \in R$ satisfy:
\[
(p_a, p_b) \in \mathcal{C}, \quad 
p_a \in T(r_i), \; p_b \in T(r_j),
\]
and both are simultaneously enabled by the same input marking $M$ in the FPN:
\[
\exists M : \text{Enabled}(r_i, M) \wedge \text{Enabled}(r_j, M).
\]
Here, ``Enabled'' denotes that all antecedent places of a transition 
have nonzero truth degree (token value) under marking $M$.
Thus, simultaneous enablement implies that two ethically conflicting rules 
could be fired under identical conditions. The reachability graph will expose conflicting transitions.

\paragraph{Detection Procedure.}
\begin{enumerate}[label=(\alph*)]
    \item Construct the reachability graph of $N$.
    \item For each reachable marking $M$, identify all enabled transitions $t_i$.
    \item For each enabled pair $(t_i, t_j)$, check if 
    $(p_a, p_b) \in \mathcal{C}$ with 
    $p_a \in T(r_i), p_b \in T(r_j)$.
    \item If such a pair exists, flag a cross-principle conflict.
\end{enumerate}


\paragraph{Principle Redundancy Verification: }
\label{subsubsec:principle_redundancy}

Two rules $r_i, r_j \in R$ are considered \emph{redundant} if they share 
identical antecedent and consequent structures and are associated 
with the same ethical principle tags:
\[
(\text{Ante}(r_i) = \text{Ante}(r_j)) \wedge
(\text{Cons}(r_i) = \text{Cons}(r_j)) \wedge
(T(r_i) = T(r_j)).
\]
In the FPN representation, redundancy corresponds to transitions 
$t_i$ and $t_j$ that have identical input and output place sets, 
equal transition weights $\beta_i = \beta_j$, and identical labels $T(r_i) = T(r_j)$.
\paragraph{Detection Procedure.}
\begin{enumerate}[label=(\alph*)]
    \item Construct the reachability graph of $N$.
    \item For each reachable marking $M$, identify all enabled transitions $t_i$.
    \item For each enabled pair $(t_i, t_j)$ where $t_i \equiv t_j$, check if 
   firing $t_i$ leads to the same marking as firing $t_j$.
    \item If they are always identical, then transitions are redundant.
\end{enumerate}



\begin{proposition}[Principle-Level Verification]
\label{prop:principle_verification}
Let \textit{fEDM+} be a fuzzy Ethical Decision-Making model 
with principle traceability function $T : R \rightarrow 2^{P}$.
Then structural verification of the corresponding FPN $N$ 
can be extended to verify the following conditions:

\begin{enumerate}[label=(\alph*)]
    \item \textbf{Principle Coverage:} 
    $\forall p \in P,\, \exists r_i \in R : p \in T(r_i)$.
    \item \textbf{Cross-Principle Consistency:} 
    $\nexists (r_i,r_j) \in R^2,\, M$ such that 
    $(p_a,p_b) \in \mathcal{C}$ and both $r_i,r_j$ are enabled under the marking $M$.
    \item \textbf{Principle Redundancy:} 
    $\nexists (r_i,r_j) \in R^2$ with identical antecedents, consequents, and principle tags.
\end{enumerate}

Violations of (a), (b), or (c) correspond respectively to 
\emph{principal incompleteness}, \emph{principal inconsistency}, 
and \emph{principal redundancy}.
\end{proposition}

\subsection{Validation of EDM Models}
\label{val}







While verification is about dealing with structural errors in the rule base, validation is about dealing with semantic errors. Two types of semantic errors can be identified \cite{4675474}:
\begin{itemize}
    \item Semantic incompleteness: This kind of error occurs if the model does not meet users' requirements and is reflected as missing rules, and/or missing antecedents or consequents in a rule from the users' point of view.
    \item Semantic incorrectness: This type of error occurs if the model produces an output that is different from the expected output for given identical input data in the validation referent. Semantic incorrectness also indicates that the model does not meet the needs of users.
\end{itemize}
To validate the created \textit{fEDM+} model, we will need to collect multiple validation referents from different domain experts against which to validate our model.

\section{A Case Study in the Healthcare Domain}
\label{case}
To illustrate the \textit{fEDM+} framework and to demonstrate how ethical risk assessment, decision making, explainability, and verification are jointly supported, we revisit the Patient Dilemma case study and reformulate it using the updated formal definitions introduced in Section \ref{fuzzyEDM}.

\textbf{Patient Dilemma Problem}: \textit{A care robot approaches her competent adult patient to give her her medicine in time and the patient rejects to take it. Should the care robot try again to change the patient’s mind or accept the patient’s decision as final?} 

The dilemma arises because, on the one hand, the care robot may not want to risk diminishing the patient’s autonomy by challenging her decision; on the other hand, the care robot may have concerns about why the patient refuses the treatment. Three of the four Principles/Duties of Biomedical Ethics are likely to be satisfied or violated in dilemmas of this type: the duty of Respect for Autonomy, the duty of Nonmaleficence and the duty of Beneficence.

In this case study, \textit{fERA} addresses the ethical risk of physical harm to the patient, which works as the basis for decision making by the care robot. In order to evaluate this ethical risk, the care robot can consider different parameters such as the severity of the health condition of the patient, the mental/psychological condition of the patient, physiological indicators of well-being, etc. All these parameters can be considered fuzzy concepts. In order to keep the case study as simple as possible, we consider only two parameters, viz. severity and mental conditions as input. Both inputs are rated on a scale between 0 and 10. The crisp values are then fuzzified, both into 3 sets. For the severity of the health condition, the fuzzy sets are: LOW, MEDIUM, HIGH. For the mental/psychological condition, the fuzzy sets are: BAD, AVERAGE, GOOD.
Starting from these two inputs, once fuzzified, \textit{fERA} calculates the risk level on a scale between 0\% and 100\%. Also for the output there are 3 fuzzy sets: LOW, MEDIUM, HIGH.
The inputs and the output are the antecedents and the consequents, respectively, of the rules employed by the \textit{fERA}. The fuzzy inference rules used to derive the output of \textit{fERA} (ethical risk) from the inputs are FERRs.R1, FERRs.R2 and FERRs.R3 in Listing \ref{lst1}.
The initial inputs, in this case severity and mental, are provided by the user (the care robot in this case). These values are then fuzzified using a membership function (MF) \cite{ross2005fuzzy}. In this paper, we choose to use the trapezoidal MF because there is an interval of input crisp values for which the membership degree to the fuzzy set is 100\%.
The fuzzified input is then processed through the above mentioned inference rules
The final output is subsequently defuzzified using the centroid method (\cite{Lee1990FuzzyLI}) to find a single crisp value which defines the output of a fuzzy set. This final value provides the level of ethical risk on the life of the patient. FERD rules use this risk value to decide on the action to take in each case. Unlike the original version of the model, each rule (FER) is now explicitly associated with one or more ethical principles, enabling downstream explainability and principled verification, see Listing \ref{lst1}. This explicit principle tagging enables the ETM module to compute, for each recommended action, a structured ethical justification showing how strongly each principle contributed to the final decision.

\begin{lstlisting}[mathescape, caption={Formal Description of the $PatientEDM^+$Model},label={lst1}]
EDMm=(PatientEDMm, ERFs, RLs, As/Ds, FERs{FERRs,FERDs}, P, T)
EDMm.ERFs={Severity, Mental}
EDMm.RLs={Risk}
EDMm.As={Actions/Decisions}
EDMm.FERs={R1,R2,...,R6}
EDMm.P={Principles}
EDM.T={Traceability_Mapping}
EDMm.ERFs.Severity={Severity, value}
EDMm.ERFs.Severity.value= {low, medium, high}
EDMm.ERFs.Mental={Mental,value}
EDMm.ERFs.Mental.value={good, average, bad}
EDMm.RLs.Risk={RiskLevel,value}
EDMm.RLs.Risk.value={low, medium,high}
EDMm.As.Actions={Action/Decision, value}
EDMm.As.Action.value={accept, tryAgainLater,tryAgainNow}
EDMm.P.Principles={Principle, value}
EDMm.T.Traceability_Mapping={T($R_i$) | $R_i \in$ EDMm.FERs, value}
EDMm.P.Principles.value = {Autonomy, Beneficence, Nonmaleficence}
EDMm.T.Traceability_Mapping.value $\subseteq$ {2^(EDMm.P.Principles.value)}
EDMm.FERs.FERRs.R1=(Rule1, (Severity(low) $\wedge$ Mental(good))$\vee$ (Severity(medium) $\wedge$ Mental(good)) $\vee$ (Severity(low) $\wedge$ Mental(average))$\vee$ (Severity(low) $\wedge$ Mental(bad)), Risk(low),0.80, {Autonomy})

EDMm.FERs.FERRs.R2=(Rule2, (Severity(high) $\wedge$ Mental(good))$\vee$ (Severity(medium) $\wedge$ Mental(average)), Risk(medium),0.70, {Beneficence} )

EDMm.FERs.FERRs.R3=(Rule3, (Severity(high) $\wedge$ Mental(average))$\vee$ (Severity(medium) $\wedge$ Mental(bad)) $\vee$ (Severity(high) $\wedge$ Mental(bad)), Risk(high),0.90, {Nonmaleficence, Beneficence} )

EDMm.FERs.FERDs.R4=(Rule4, Risk(low), Action(accept),0.80, {Autonomy} )
EDMm.FERs.FERDs.R5=(Rule5, Risk(high), Action(tryAgainNow),0.90, {Beneficence, Nonmaleficence} )
EDMm.FERs.FERDs.R6=(Rule6, Risk(medium), Action(tryAgainLater),0.70, {Beneficence, Autonomy} )
EDMm.T.Traceability_Mapping.T($R_1$)={Autonomy}
EDMm.T.Traceability_Mapping.T($R_2$)={Beneficence}
EDMm.T.Traceability_Mapping.T($R_3$)={Nonmaleficence, Beneficence}
EDMm.T.Traceability_Mapping.T($R_4$)={Autonomy}
EDMm.T.Traceability_Mapping.T($R_5$)={Beneficence, Nonmaleficence}
EDMm.T.Traceability_Mapping.T($R_6$)={Beneficence, Autonomy}
\end{lstlisting}

The \textit{fEDM+}  model shown in Listing \ref{lst1} has been completely implemented in Python. A screenshot of the model output is shown in Figure \ref{fig:screenshot}. The automation of verification and validation processes (or parts of them) is subject to our future work.
\begin{figure}[ht]
  \centering
  \includegraphics[width=\linewidth]{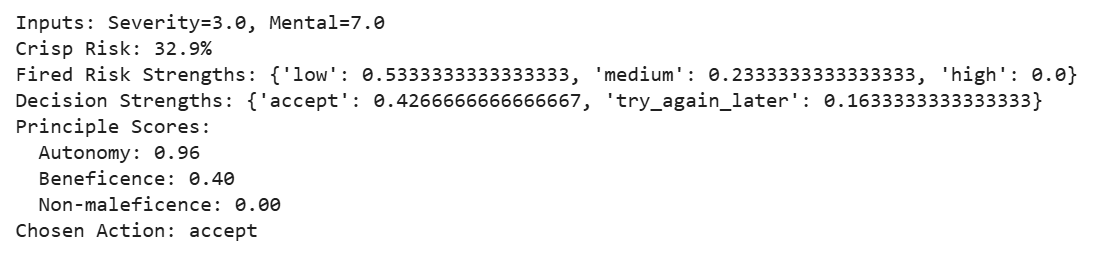}
  \caption{System output for the $PatientEDM^+$ case study}
  \label{fig:screenshot}
  \vspace{1cm}
\end{figure}

\subsection{Verification Process}
The steps we follow in the verification process of the rule-base are the ones reported in Section \ref{ver}.
In particular, by applying rule normalization, Rule1 will result in four rules, Rule2 in two rules, Rule3 in three rules, Rules 4,5, and 6 in one rule each. The total number of resulting rules after normalization is 12, see Listing \ref{lst2}.

\begin{lstlisting}[mathescape, caption={Normalized rules},label={lst2}]
(Severity(low) $\wedge$ Mental(good))$\rightarrow$ Risk(low)
(Severity(medium) $\wedge$ Mental(good))$\rightarrow$ Risk(low)
(Severity(low) $\wedge$ Mental(average))$\rightarrow$ Risk(low)
(Severity(low) $\wedge$ Mental(bad)) $\rightarrow$ Risk(low)
(Severity(high) $\wedge$ Mental(good))$\rightarrow$ Risk(medium)
(Severity(medium) $\wedge$ Mental(average))$\rightarrow$ Risk(medium)
(Severity(high) $\wedge$ Mental(average)) $\rightarrow$ Risk(high)
(Severity(medium) $\wedge$ Mental(bad)) $\rightarrow$ Risk(high)
(Severity(high) $\wedge$ Mental(bad)) $\rightarrow$ Risk(high)
Risk(low) $\rightarrow$ Action(accept)
Risk(high) $\rightarrow$ Action(tryAgainNow)
Risk(medium) $\rightarrow$ Action(tryAgainLater)
\end{lstlisting}

The transformation of this rule base to FPN will result in the FPN shown in Figure \ref{fig:petri}. We have twelve places and twelve transitions. The transitions $t_1 ... t_{12}$ correspond to the rules in Listing \ref{lst2} in the same order. While $P_1 ... P_{12}$ corresponds to: severity(low), severity(medium), severity(high), mental(good), mental(average), mental(bad), risk(low), risk(medium), risk(high), action(accept), action(try\_again\_later), action(try\_again\_now) respectively.

To check for structural errors, we need to generate the reachability graph of the resulting FPN. Each node in the reachability graph is a marking of the FPN model. Each directed edge represents transition firing and connects one node to the other. For generating the reachability graph, we adopt the algorithm presented in \cite{1185845}.
After generating the reachability graph, structural errors including incompleteness, inconsistency, redundancy , and circularity are checked and eliminated.

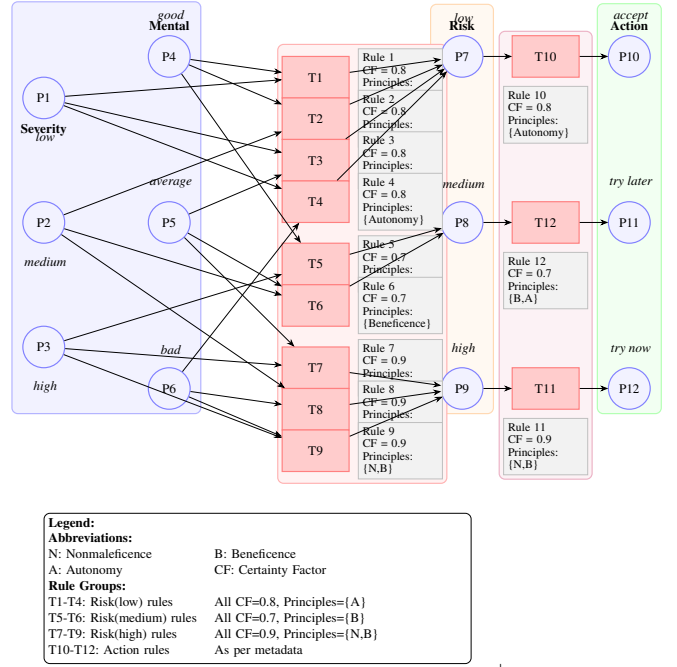
\begin{figure}
\centering
\resizebox{\linewidth}{!}{
\begin{tikzpicture}[
    node distance=1cm and 2cm,
    place/.style={circle, draw=blue!50, fill=blue!5, thick, minimum size=10mm},
    trans/.style={rectangle, draw=red!50, fill=red!20, thick, 
                  minimum width=16mm, minimum height=10mm, align=center},
    meta/.style={rectangle, draw=gray!50, fill=gray!10, thin, 
                 font=\footnotesize, text width=18mm, align=left},
    label/.style={font=\small\bfseries},
    var/.style={font=\normalsize\itshape},
    >=Stealth
]

\begin{scope}[local bounding box=severity]
    \node[place] (P1) at (0,6) {P1};
    \node[below=0 of P1, label] {Severity};
    \node[below=0.2 of P1, var] {low};
    
    \node[place] (P2) at (0,3) {P2};
    \node[below=0.2 of P2, var] {medium};
    
    \node[place] (P3) at (0,0) {P3};
    \node[below=0.2 of P3, var] {high};
\end{scope}

\begin{scope}[local bounding box=mental]
    \node[place] (P4) at (3,7) {P4};
    \node[above=0 of P4, label] {Mental};
    \node[above=0.2 of P4, var] {good};
    
    \node[place] (P5) at (3,3) {P5};
    \node[above=0.2 of P5, var] {average};
    
    \node[place] (P6) at (3,-1) {P6};
    \node[above=0.2 of P6, var] {bad};
\end{scope}

\begin{scope}[local bounding box=risk]
    \node[place] (P7) at (10,7) {P7};
    \node[above=0 of P7, label] {Risk};
    \node[above=0.2 of P7, var] {low};
    
    \node[place] (P8) at (10,3) {P8};
    \node[above=0.2 of P8, var] {medium};
    
    \node[place] (P9) at (10,-1) {P9};
    \node[above=0.2 of P9, var] {high};
\end{scope}

\begin{scope}[local bounding box=action]
    \node[place] (P10) at (14,7) {P10};
    \node[above=0 of P10, label] {Action};
    \node[above=0.2 of P10, var] {accept};
    
    \node[place] (P11) at (14,3) {P11};
    \node[above=0.2 of P11, var] {try later};
    
    \node[place] (P12) at (14,-1) {P12};
    \node[above=0.2 of P12, var] {try now};
\end{scope}

\begin{scope}[local bounding box=lowRisk]
    \node[trans] (T1) at (6.5,6.5) {T1};
    \node[meta, right=0.2 of T1, anchor=west] (M1) {
        Rule 1\\
        CF = 0.8\\
        Principles: \{Autonomy\}
    };
    
    \node[trans] (T2) at (6.5,5.5) {T2};
    \node[meta, right=0.2 of T2, anchor=west] (M2) {
        Rule 2\\
        CF = 0.8\\
        Principles: \{Autonomy\}
    };
    
    \node[trans] (T3) at (6.5,4.5) {T3};
    \node[meta, right=0.2 of T3, anchor=west] (M3) {
        Rule 3\\
        CF = 0.8\\
        Principles: \{Autonomy\}
    };
    
    \node[trans] (T4) at (6.5,3.5) {T4};
    \node[meta, right=0.2 of T4, anchor=west] (M4) {
        Rule 4\\
        CF = 0.8\\
        Principles: \{Autonomy\}
    };
\end{scope}

\begin{scope}[local bounding box=mediumRisk]
    \node[trans] (T5) at (6.5,2) {T5};
    \node[meta, right=0.2 of T5, anchor=west] (M5) {
        Rule 5\\
        CF = 0.7\\
        Principles: \{Beneficence\}
    };
    
    \node[trans] (T6) at (6.5,1) {T6};
    \node[meta, right=0.2 of T6, anchor=west] (M6) {
        Rule 6\\
        CF = 0.7\\
        Principles: \{Beneficence\}
    };
\end{scope}

\begin{scope}[local bounding box=highRisk]
    \node[trans] (T7) at (6.5,-0.5) {T7};
    \node[meta, right=0.2 of T7, anchor=west] (M7) {
        Rule 7\\
        CF = 0.9\\
        Principles: \{N,B\}
    };
    
    \node[trans] (T8) at (6.5,-1.5) {T8};
    \node[meta, right=0.2 of T8, anchor=west] (M8) {
        Rule 8\\
        CF = 0.9\\
        Principles: \{N,B\}
    };
    
    \node[trans] (T9) at (6.5,-2.5) {T9};
    \node[meta, right=0.2 of T9, anchor=west] (M9) {
        Rule 9\\
        CF = 0.9\\
        Principles: \{N,B\}
    };
\end{scope}

\begin{scope}[local bounding box=actionRules]
    \node[trans] (T10) at (12,7) {T10};
    \node[meta, below=0.2 of T10] (M10) {
        Rule 10\\
        CF = 0.8\\
        Principles: \{Autonomy\}
    };
    
    \node[trans] (T11) at (12,-1) {T11};
    \node[meta, below=0.2 of T11] (M11) {
        Rule 11\\
        CF = 0.9\\
        Principles: \{N,B\}
    };
    
    \node[trans] (T12) at (12,3) {T12};
    \node[meta, below=0.2 of T12] (M12) {
        Rule 12\\
        CF = 0.7\\
        Principles: \{B,A\}
    };
\end{scope}

\foreach \i/\j in {1/1, 1/3, 1/4, 2/2, 2/6, 2/8, 3/5, 3/7, 3/9} {
    \draw[->] (P\i) -- (T\j);
}

\foreach \i/\j in {4/1, 4/2, 4/5, 5/3, 5/6, 5/7, 6/4, 6/8, 6/9} {
    \draw[->] (P\i) -- (T\j);
}

\foreach \i/\j in {1/7, 2/7, 3/7, 4/7, 5/8, 6/8, 7/9, 8/9, 9/9} {
    \draw[->] (T\i) -- (P\j);
}

\draw[->] (P7) -- (T10);
\draw[->] (P8) -- (T12);
\draw[->] (P9) -- (T11);

\draw[->] (T10) -- (P10);
\draw[->] (T12) -- (P11);
\draw[->] (T11) -- (P12);

\begin{scope}[on background layer]
    \node[draw=blue!30, fill=blue!5, rounded corners, thick, 
          fit=(severity) (mental), label={[label distance=0.2cm]above:Input Variables}] {};
    
    \node[draw=orange!30, fill=orange!5, rounded corners, thick,
          fit=(risk), label={[label distance=0.2cm]above:Intermediate Variable}] {};
    
    \node[draw=green!30, fill=green!5, rounded corners, thick,
          fit=(action), label={[label distance=0.2cm]above:Output Variable}] {};
    
    \node[draw=red!30, fill=red!5, rounded corners, thick,
          fit=(lowRisk) (mediumRisk) (highRisk), 
          label={[label distance=0.2cm]above:Risk Assessment Rules (1-9)}] {};
    
    \node[draw=purple!30, fill=purple!5, rounded corners, thick,
          fit=(actionRules), 
          label={[label distance=0.2cm]above:Action Decision Rules (10-12)}] {};
\end{scope}

\node[draw, rounded corners, fill=white, anchor=north west, text width=10cm] at (0,-4) {
    \textbf{Legend:}\\
    \begin{tabular}{@{}ll@{}}
        \textbf{Abbreviations:} & \\
        N: Nonmaleficence & B: Beneficence \\
        A: Autonomy & CF: Certainty Factor \\
        \textbf{Rule Groups:} & \\
        T1-T4: Risk(low) rules & All CF=0.8, Principles=\{A\} \\
        T5-T6: Risk(medium) rules & All CF=0.7, Principles=\{B\} \\
        T7-T9: Risk(high) rules & All CF=0.9, Principles=\{N,B\} \\
        T10-T12: Action rules & As per metadata \\
    \end{tabular}
};

\end{tikzpicture}
}
\caption{The FPN of the $PatientEDM^+$ model}
\label{fig:petri}
\vspace{1cm}
\end{figure}


The corresponding reachability graph for the above FPN is shown in Figure \ref{fig:reach}. In the initial marking, places $P_1$ and $P_4$ are regarded as true antecedents, and initially filled (set to $1$) for this reason, the rest of places are set to $0$, that is why in the first node there are two one's. This causes $t_1$ to fire in the first step. After firing this transition, in the second step, place $P_7$ is filled and the corresponding value in the node vector is set to $1$. In the final step, by firing $t_{10}$ (the enabled transition), place $P_{10}$ will be filled up.

All the places and transitions exist, so there are no incompleteness errors. $P_1$, $P_2$, and $P_3$ are different states of one property (severity). Similarly,  $P_4, P_5, P_6$ are different states of one property (mental), $P_7, P_8, P_9$ are different states of one property (risk), $P_{10}, P_{11}, P_{12}$ are different states of one property (action), Their simultaneous existence may refer to some concept of inconsistency in the rule base. The reachability graph has no loops means no circularity errors and finally having no transitions underlined, speaks for non-redundancy.


Regarding the principles' coverage verification, by scanning the FPN, we can construct the following principles coverage vector: 
\begin{multline*}
    C = [\text{coverage(Autonomy)}, \text{coverage(Beneficence)}, \\
         \text{coverage(Nonmaleficence)}] = [1, 1, 1]
\end{multline*}
that proves that the model satisfies the principle coverage completeness.
Regarding cross-principle consistency verification, in our case study we have the following principles pair [Autonomy, Nonmaleficence] is an incompatible pair. Looking at the reachability graph, we can see that no marking $M$ exists where a rule tagged Autonomy and a rule tagged Non-maleficence are simultaneously enabled with incompatible consequents. This indicates that no cross-principle conflict is detected.

Finally, no duplicate rules with identical semantic content were found. Example: $R_1$ and $R_2$ differ in both consequents and principles. $R_4, R_5, R_6$ have distinct risk preconditions and/or actions. Thus, no principle redundancy is detected.


\begin{figure}[h]
\centering
\resizebox{\linewidth}{!}{
\begin{tikzpicture}[
    marking/.style={rectangle, draw=black!60, fill=white, thick, 
                    minimum width=2cm, minimum height=0.6cm,
                    align=center, font=\footnotesize},
    initial/.style={marking, draw=green!60, fill=green!5},
    intermediate/.style={marking, draw=orange!60, fill=orange!5},
    final/.style={marking, draw=red!60, fill=red!5},
    transition/.style={font=\footnotesize\itshape, midway, above, sloped},
    vector/.style={font=\tiny, below, text width=2.5cm, align=center},
    >=Stealth
]

\node[initial] (M1) at (0,5) {M1};
\node[vector] at (M1.south) {[1,0,0,1,0,0,0,0,0,0,0,0]};

\node[initial] (M2) at (0,4) {M2};
\node[vector] at (M2.south) {[1,0,0,0,1,0,0,0,0,0,0,0]};

\node[initial] (M3) at (0,3) {M3};
\node[vector] at (M3.south) {[1,0,0,0,0,1,0,0,0,0,0,0]};

\node[initial] (M4) at (0,2) {M4};
\node[vector] at (M4.south) {[0,1,0,1,0,0,0,0,0,0,0,0]};

\node[initial] (M5) at (0,1) {M5};
\node[vector] at (M5.south) {[0,1,0,0,1,0,0,0,0,0,0,0]};

\node[initial] (M6) at (0,0) {M6};
\node[vector] at (M6.south) {[0,1,0,0,0,1,0,0,0,0,0,0]};

\node[initial] (M7) at (0,-1) {M7};
\node[vector] at (M7.south) {[0,0,1,1,0,0,0,0,0,0,0,0]};

\node[initial] (M8) at (0,-2) {M8};
\node[vector] at (M8.south) {[0,0,1,0,1,0,0,0,0,0,0,0]};

\node[initial] (M9) at (0,-3) {M9};
\node[vector] at (M9.south) {[0,0,1,0,0,1,0,0,0,0,0,0]};

\node[intermediate] (M10) at (4,4) {M10};
\node[vector] at (M10.south) {[0,0,0,0,0,0,1,0,0,0,0,0]};

\node[intermediate] (M11) at (4,1) {M11};
\node[vector] at (M11.south) {[0,0,0,0,0,0,0,1,0,0,0,0]};

\node[intermediate] (M12) at (4,-2) {M12};
\node[vector] at (M12.south) {[0,0,0,0,0,0,0,0,1,0,0,0]};

\node[final] (M13) at (8,4) {M13};
\node[vector] at (M13.south) {[0,0,0,0,0,0,0,0,0,1,0,0]};

\node[final] (M14) at (8,1) {M14};
\node[vector] at (M14.south) {[0,0,0,0,0,0,0,0,0,0,1,0]};

\node[final] (M15) at (8,-2) {M15};
\node[vector] at (M15.south) {[0,0,0,0,0,0,0,0,0,0,0,1]};

\draw[->, thick] (M1) -- node[transition] {T1} (M10);
\draw[->, thick] (M2) -- node[transition] {T3} (M10);
\draw[->, thick] (M3) -- node[transition] {T4} (M10);
\draw[->, thick] (M4) -- node[transition] {T2} (M10);

\draw[->, thick] (M5) -- node[transition] {T6} (M11);
\draw[->, thick] (M7) -- node[transition] {T5} (M11);

\draw[->, thick] (M6) -- node[transition] {T8} (M12);
\draw[->, thick] (M8) -- node[transition] {T7} (M12);
\draw[->, thick] (M9) -- node[transition] {T9} (M12);

\draw[->, thick] (M10) -- node[transition] {T10} (M13);
\draw[->, thick] (M11) -- node[transition] {T12} (M14);
\draw[->, thick] (M12) -- node[transition] {T11} (M15);

\node[draw=black, fill=white, rounded corners, anchor=north west, text width=10cm, font=\small] 
    at (-1, -4) {
    \textbf{Vector Key:} [P1,P2,P3,P4,P5,P6,P7,P8,P9,P10,P11,P12] \\where 1=token present, 0=no token. 
   \\P1-3: Severity(low, medium, high), \\P4-6: Mental(good, average, bad), \\P7-9: Risk(low, medium, high), \\P10-12: Action(accept, try later, try now)
    
};

\end{tikzpicture}
}
\caption{Reachability Graph of the FPN of the $PatientEDM^+$ model}
\label{fig:reach}
\vspace{1cm}
\end{figure}

\subsection{Validation Process}
To validate the created EDM model, we will need to construct validation referents. The validation referents provide a standard against which to compare our EDM models to determine the validity of these models for our needs. A validation referent gives the key information that a model must represent. Validation referents should be developed by experienced domain experts and expressed in the same language that is used to create the EDM model.

To build the validation referents, validation  evidences are collected from domain experts and recorded in the validation referents. There exists three types of validation evidences (rules): Type1) $ERFs \xrightarrow{}RLs$, Type2) $RLs \xrightarrow{} AS/DS$, Type3) $ERFs \xrightarrow{} AS/DS$.
To ensure the validity of the EDM models, these types of rules above should be deduced from the EDM models. Validation will be carried out in two stages: first) static validation, and Second) dynamic validation. For this case study and to demonstrate how to carry the validation process, we constructed three validation referents, one by the patient advocate, another by the medical doctor, and the third by the hospital authority.

Listings \ref{lst3},\ref{lst4}, and \ref{lst5} shows the three validation referents for our case study. Listing \ref{lst3} represents the patient-advocate referent which strongly privileges autonomy. It expects the agent to accept competent refusals except where imminent, severe, and high-risk harm is clearly indicated.
Listing \ref{lst4} represents the clinician referent. This referent prioritizes preventing harm and restoring health, willing to re-attempt if risk is medium or higher. This referent is somewhat more tolerant of deviations, and expects higher readiness to intervene.
Listing \ref{lst5} is the hospital ethics board referent.
 This referent balances autonomy and beneficence, places explicit weight on long-term consequences and legal/safety constraints. It tolerates modest deviations but requires that principle-score alignment be present for acceptance. Useful for compliance \& regulatory validation.

$LTconsequences$ refer to the long term consequences that might result from non taking the medication by the patient. $RR1$ and $RR2$ are the reasoning rules, with reference values given by the domain experts, that will serve for the dynamic validation. $\alpha$ is used to refer to the degree of truth of a proposition.

\begin{lstlisting}[mathescape, style=mystyle,caption={Referent \(V_1\) — Patient-Advocate / Autonomy-Focused},label={lst3}]
$V_1 = (M_1, \Pi_1, \rho_1, \Theta_1, \tau_1)$

P = {Autonomy, Beneficence, Nonmaleficence}

=== Referential metadata ===
$M1 : E^* \rightarrow A^*$ (given by the rule set below)
Principle ordering ($\Pi_1$): Autonomy $\succ$ Beneficence $\succ$ Nonmaleficence
Risk tolerance ($\rho_1$): $\rho_1=0.80$ (i.e. tolerant-intervention only when risk very high)
Semantic tolerance ($\tau_1$):$\tau_1=0.25$ (more forgiving)
Acceptable action mapping ($\Theta$):
if $r > 0.80 \rightarrow${tryAgainNow}
else if $r\geq 0.50 \leq 0.80 \rightarrow${tryAgainLater} (prefer delay over immediate override)
else $\rightarrow$ {accept}

ERFs = { Severity, Mental, LTconsequences } with values {low, medium, high}
RLs = { Risk } with values {low, medium, high}
Actions = { accept, tryAgainLater, tryAgainNow }

=== Referential FERRs (risk-assessment expectations; M1's RL mapping) ===
M1.FERR.R1 = (P_R1, (Severity(low) $\wedge$ Mental(good)), Risk(low), 0.90, {Autonomy})
M1.FERR.R2 = (P_R2, (Severity(medium) $\wedge$ Mental(good)), Risk(low), 0.85, {Autonomy})
M1.FERR.R3 = (P_R3, (Severity(low) $\wedge$ Mental(average)), Risk(low), 0.80, {Autonomy})
M1.FERR.R4 = (P_R4, (Severity(medium) $\wedge$ Mental(average)), Risk(medium), 0.70, {Beneficence})
M1.FERR.R5 = (P_R5, (Severity(high) $\wedge$ Mental(good)), Risk(medium), 0.70, {Beneficence})
M1.FERR.R6 = (P_R6, (Severity(high) $\wedge$ Mental(average)) $\vee$ (Severity(medium) $\wedge$ Mental(bad)), 
              Risk(high), 0.95, {Nonmaleficence})

=== Referential FERDs (decision expectations) ===
M1.FERD.D1 = (P_D1, Risk(low) $\wedge$ LTconsequences(low), Action(accept), 0.95, {Autonomy})
M1.FERD.D2 = (P_D2, Risk(medium) $\wedge$ LTconsequences(medium), Action(tryAgainLater), 
              0.70, {Beneficence, Autonomy})
M1.FERD.D3 = (P_D3, Risk(high) $\vee$ LTconsequences(high), Action(tryAgainNow), 
              0.90, {Nonmaleficence, Beneficence})

=== Traceability mapping T1 ===
T1(P_R1)={Autonomy}, T1(P_R2)={Autonomy}, T1(P_R3)={Autonomy}
T1(P_R4)={Beneficence}, T1(P_R5)={Beneficence}, T1(P_R6)={Nonmaleficence}
T1(P_D1)={Autonomy}, T1(P_D2)={Beneficence,Autonomy}, T1(P_D3)={Nonmaleficence,Beneficence}

=== Validation Rules (dynamic reference checks) ===
RR1_V1: $\alpha$(Severity(medium)) $= 0.95  \wedge  \alpha$(Mental(bad)) $= 0.9  \rightarrow  \alpha$(Risk(high)) $> 0.80$.

RR2_V1: $\alpha$(Severity(low)) $= 0.9   \wedge  \alpha$(Mental(good)) $= 0.8 \wedge \alpha$(LTconsequences(low)) $= 0.85 \rightarrow \alpha$(Action(accept)) $> 0.70$

\end{lstlisting}

\begin{lstlisting}[mathescape, style=mystyle,caption={Referent \(V_2\) — Clinician / Frontline Nurse (beneficence + nonmaleficence focus)},label={lst4}]
$V_2 = (M_2, \Pi_2, \rho_2, \Theta_2, \tau_2)$

P = {Autonomy, Beneficence, Nonmaleficence}

=== Referential metadata ===
$M1 : E^* \rightarrow A^*$ (given by the rule set below)
Principle ordering ($\Pi_2$): Nonmaleficence $\succ$ Beneficence $\succ$ Autonomy

Risk tolerance ($\rho_2$): $\rho_1=0.60$ (i.e. clinician treats risk > 60% as high)

Semantic tolerance ($\tau_2$): $\tau_2=0.15$ (strict; require high agreement)

Acceptable action mapping ($\Theta_2$) (deterministic):
if $r \geq 0.60$ then accept only {tryAgainNow}
else if $r \geq 0.50$ and $< 0.60$ then {tryAgainLater}
else {accept}

ERFs = {Severity, Mental, LTconsequences} with values {low, medium, high}
RLs = {Risk} with values {low, medium, high}
Actions = {accept, tryAgainLater, tryAgainNow}

=== Referential FERRs (risk expectations) ===
M2.FERR.R1 = (C_R1, (Severity(low) $\wedge$ Mental(good)), Risk(low), 0.75, {Autonomy})
M2.FERR.R2 = (C_R2, (Severity(medium) $\wedge$ Mental(average) $\vee$ Mental(bad)), Risk(medium), 0.85, {Beneficence})
M2.FERR.R3 = (C_R3, (Severity(high) $\wedge$ (Mental(average) $\vee$ Mental(bad))), Risk(high), 0.95, {Nonmaleficence})
M2.FERR.R4 = (C_R4, (Mental(bad) $\wedge$ Severity(low)), Risk(medium), 0.88, {Beneficence,Nonmaleficence})

=== Referential FERDs (decision expectations) ===
M2.FERD.D1 = (C_D1, Risk(low) $\wedge$ LTconsequences(low), Action(accept), 0.70, {Autonomy})
M2.FERD.D2 = (C_D2, Risk(medium), Action(tryAgainLater), 0.80, {Beneficence})
M2.FERD.D3 = (C_D3, Risk(high), Action(tryAgainNow), 0.95, {Beneficence,Nonmaleficence})

=== Traceability mapping T2 ===
T2(C_R1)={Autonomy}, T2(C_R2)={Beneficence}, T2(C_R3)={Nonmaleficence}
T2(C_R4)={Beneficence,Nonmaleficence}
T2(C_D1)={Autonomy}, T2(C_D2)={Beneficence}, T2(C_D3)={Beneficence,Nonmaleficence}

=== Validation Rules (dynamic reference checks) ===
RR1_V2: $\alpha$(Severity(high)) $= 0.85 \wedge \alpha$(Mental(average)) $= 0.8  \rightarrow \alpha$(Risk(high)) $> 0.70$

RR2_V2: $\alpha$(Risk(medium)) $= 0.8 \wedge \alpha$(LTconsequences(medium)) $= 0.85 \rightarrow \alpha$(Action(tryAgainLater)) $> 0.70$

\end{lstlisting}
\begin{lstlisting}[mathescape, style=mystyle,caption={Referent \(V_3\) — Hospital Ethics Board / Policy Authority (balanced, LTconsequences emphasized)},label={lst5}]
$V_3 = (M_3, \Pi_3, \rho_3, \Theta_3, \tau_3)$

P = {Autonomy, Beneficence, Nonmaleficence}

=== Referential metadata ===
$M1 : E^* \rightarrow A^*$ (given by the rule set below)

Principle ordering ($\Pi_3$): Beneficence $\succ$ Non-maleficence $\succ$ Autonomy

Risk tolerance ($\rho_3$): $\rho_3=0.70$

Semantic tolerance ($\tau_3$): $\tau_3=0.20$

Acceptable action mapping ($\Theta_3$) (banded):
if $r > 0.70$ $\rightarrow$ {tryAgainNow}
if $0.50 \leq r \leq 0.70$ $\rightarrow$ {tryAgainLater}
else $\rightarrow$ {accept}

ERFs = {Severity, Mental, LTconsequences} with values {low, medium, high}
RLs = {Risk} with values {low, medium, high}
Actions = {accept, tryAgainLater, tryAgainNow}

=== Referential FERRs (risk expectations) ===
M3.FERR.R1 = (H_R1, (Severity(low) $\wedge$ Mental(good)), Risk(low), 0.85, {Autonomy})
M3.FERR.R2 = (H_R2, (Severity(medium) $\wedge$ Mental(average)), Risk(medium), 0.80, {Beneficence})
M3.FERR.R3 = (H_R3, (Severity(high) $\wedge$ Mental(bad)), Risk(high), 0.92, {Nonmaleficence})

=== Referential FERDs (decision expectations) ===
M3.FERD.D1 = (H_D1, Risk(low) $\wedge$ LTconsequences(low), Action(accept), 0.88, {Autonomy})
M3.FERD.D2 = (H_D2, Risk(medium) $\wedge$ LTconsequences(medium), Action(tryAgainLater), 0.78, {Beneficence,Autonomy})
M3.FERD.D3 = (H_D3, Risk(high) $\vee$ LTconsequences(high), Action(tryAgainNow), 0.94, {Nonmaleficence,Beneficence})
M3.FERD.D4 = (H_D4, Risk(low) $\wedge$ LTconsequences(medium)), Action(tryAgainLater), 0.72, {Beneficence})

=== Traceability mapping T3 ===
T3(H_R1)={Autonomy}, T3(H_R2)={Beneficence}, T3(H_R3)={Nonmaleficence}
T3(H_D1)={Autonomy}, T3(H_D2)={Beneficence,Autonomy}, T3(H_D3)={Nonmaleficence,Beneficence}, T3(H_D4)={Beneficence}

=== Validation Rules (dynamic reference checks) ===
RR1_V3: $\alpha$(Risk(high)) $= 0.85 \rightarrow \alpha$(Action(TryAgainNow) $> 0.75$

RR2_V3: $\alpha$(LTconsequences(high)) $= 0.8 \rightarrow \alpha$(Action(TryAgainNow)) $> 0.75$

\end{lstlisting}
\paragraph{Static Validation (Semantic Completeness):}
In this stage, we are going to search the FPN that represents the fEDM model and compare the searching results with the three validation referents to determine if our model is semantically incomplete. For conducting static validation, we are going to follow this procedure \cite{4675474}:
\begin{itemize}
    \item Step1: searching the places of the FPN that represents the EDM model.
    \item Step2: recording the properties that correspond to each place into the corresponding $ERFs, RLs, As/Ds$. 
    \item Step3: compare the searched $ERFs, RLs, As/Ds$ with those of the validation referents models $ERFsr, RLsr, Asr/Dsr$.
    \item Step4: if the number of the searched $ERFs$ is less than those of a referent model then the model may misses antecedents. If the number of the searched $As/Ds$ is less than those of a referent model then the model may misses consequents. If the number of the searched $RLs$ is less than those of a referent model then the model may misses antecedents or consequents. If an expected rule doesnot appear in the FPN, then the model might misses rules.
\end{itemize}

By applying the above steps on our $PatientEDM^+$ model, we found that: the property $LTconsequences$ which is an ethically relevant fact (mentioned by all three referent models) does not appear in the FPN of our $PatientEDM$ model, and also the rules \textit{C_R4, P_D1, P_D2, P_D3, C_D1, H_D1, H_D2, H_D3, H_D4} are missing. Therefore, we can conclude that our model is semantically incomplete with respect to the referent models. By adding these missing components we render our model semantically complete with respect to the validation referent models presented in listings \ref{lst3}, \ref{lst4}, \ref{lst5} and discussed with the domain experts.
\begin{figure}[h]
\centering
\resizebox{\linewidth}{!}{%
\begin{tikzpicture}[
    node distance=1.2cm and 3cm,
    place/.style={circle, draw=blue!50, fill=blue!10, thick, minimum size=10mm},
    trans/.style={rectangle, draw=red!50, fill=red!10, thick,
                  minimum width=16mm, minimum height=10mm, align=center},
    meta/.style={rectangle, draw=black!30, fill=white, thin,
                 font=\footnotesize, text width=28mm, align=left, inner sep=3pt},
    label/.style={font=\small\bfseries},
    var/.style={font=\normalsize},
    >=Stealth
]

\node[place] (P1) at (0,10) {P1};
\node[label, above=0 of P1] {Severity};
\node[below=0.2 of P1, var] {low};

\node[place] (P2) at (0,7) {P2};
\node[below=0.2 of P2, var] {medium};

\node[place] (P3) at (0,4) {P3};
\node[below=0.2 of P3, var] {high};

\node[place] (P4) at (4,12) {P4};
\node[label, above=0 of P4] {Mental};
\node[below=0.2 of P4, var] {good};

\node[place] (P5) at (4,8) {P5};
\node[below=0.2 of P5, var] {average};

\node[place] (P6) at (4,4) {P6};
\node[below=0.2 of P6, var] {bad};

\node[place] (P7) at (8,10) {P7};
\node[label, above=0 of P7] {LTconsequences};
\node[below=0.2 of P7, var] {low};

\node[place] (P8) at (8,7) {P8};
\node[below=0.2 of P8, var] {medium};

\node[place] (P9) at (8,4) {P9};
\node[below=0.2 of P9, var] {high};

\node[trans] (T1) at (10.5,11) {T1};
\node[meta, right=0.5 of T1, anchor=west] {R1a\\CF=0.8\\\{A\}};
\node[trans] (T2) at (10.5,9) {T2};
\node[meta, right=0.5 of T2, anchor=west] {R1b\\CF=0.8\\\{A\}};
\node[trans] (T3) at (10.5,7) {T3};
\node[meta, right=0.5 of T3, anchor=west] {R1c\\CF=0.8\\\{A\}};

\node[trans] (T4) at (10.5,5) {T4};
\node[meta, right=0.5 of T4, anchor=west] {R2a\\CF=0.7\\\{B\}};
\node[trans] (T5) at (10.5,3) {T5};
\node[meta, right=0.5 of T5, anchor=west] {R2b\\CF=0.7\\\{B\}};
\node[trans] (T6) at (10.5,1) {T6};
\node[meta, right=0.5 of T6, anchor=west] {R2c\\CF=0.7\\\{B\}};

\node[trans] (T7) at (10.5,-1) {T7};
\node[meta, right=0.5 of T7, anchor=west] {R3a\\CF=0.9\\\{N,B\}};
\node[trans] (T8) at (10.5,-3) {T8};
\node[meta, right=0.5 of T8, anchor=west] {R3b\\CF=0.9\\\{N,B\}};
\node[trans] (T9) at (10.5,-5) {T9};
\node[meta, right=0.5 of T9, anchor=west] {R3c\\CF=0.9\\\{N,B\}};

\node[place] (P10) at (17,11) {P10};
\node[label, above=0 of P10] {Risk};
\node[below=0.2 of P10, var] {low};

\node[place] (P11) at (17,7) {P11};
\node[below=0.2 of P11, var] {medium};

\node[place] (P12) at (17,3) {P12};
\node[below=0.2 of P12, var] {high};

\node[trans] (T10) at (20.5,11) {T10};
\node[meta, right=0.5 of T10, anchor=west] {R4\\CF=0.8\\\{A\}};
\node[trans] (T11) at (20.5,9) {T11};
\node[meta, right=0.5 of T11, anchor=west] {R5\\CF=0.9\\\{N,B\}};
\node[trans] (T12) at (20.5,7) {T12};
\node[meta, right=0.5 of T12, anchor=west] {R6\\CF=0.7\\\{B,A\}};

\node[trans] (T13) at (20.5,5) {T13};
\node[meta, right=0.5 of T13, anchor=west] {R7\\CF=0.88\\\{A\}};
\node[trans] (T14) at (20.5,3) {T14};
\node[meta, right=0.5 of T14, anchor=west] {R8\\CF=0.78\\\{B,A\}};
\node[trans] (T17) at (20.5,1) {T17};
\node[meta, right=0.5 of T17, anchor=west] {R10\\CF=0.72\\\{B\}};

\node[trans] (T15) at (20.5,-1) {T15};
\node[meta, right=0.5 of T15, anchor=west] {R9a\\CF=0.94\\\{N,B\}};
\node[trans] (T16) at (20.5,-3) {T16};
\node[meta, right=0.5 of T16, anchor=west] {R9b\\CF=0.94\\\{N,B\}};

\node[place] (P13) at (26,11) {P13};
\node[label, above=0 of P13] {Action};
\node[below=0.2 of P13, var] {accept};

\node[place] (P14) at (26,7) {P14};
\node[below=0.2 of P14, var] {try later};

\node[place] (P15) at (26,3) {P15};
\node[below=0.2 of P15, var] {try now};

\draw[->, draw=red!80!black, thick] (P1) to[out=10, in=180] (T1);
\draw[->, draw=red!80!black, thick] (P1) to[out=0, in=180] (T3);
\draw[->, draw=red!80!black, thick] (P1) to[out=-10, in=180] (T6);

\draw[->, draw=red!80!black, thick] (P2) to[out=10, in=180] (T2);
\draw[->, draw=red!80!black, thick] (P2) to[out=0, in=180] (T5);
\draw[->, draw=red!80!black, thick] (P2) to[out=-10, in=180] (T8);

\draw[->, draw=red!80!black, thick] (P3) to[out=10, in=180] (T4);
\draw[->, draw=red!80!black, thick] (P3) to[out=0, in=180] (T7);
\draw[->, draw=red!80!black, thick] (P3) to[out=-10, in=180] (T9);

\draw[->, draw=blue!80!black, thick] (P4) to[out=0, in=180] (T1);
\draw[->, draw=blue!80!black, thick] (P4) to[out=-5, in=180] (T2);
\draw[->, draw=blue!80!black, thick] (P4) to[out=-10, in=180] (T4);

\draw[->, draw=blue!80!black, thick] (P5) to[out=5, in=180] (T3);
\draw[->, draw=blue!80!black, thick] (P5) to[out=0, in=180] (T5);
\draw[->, draw=blue!80!black, thick] (P5) to[out=-5, in=180] (T7);

\draw[->, draw=blue!80!black, thick] (P6) to[out=5, in=180] (T6);
\draw[->, draw=blue!80!black, thick] (P6) to[out=0, in=180] (T8);
\draw[->, draw=blue!80!black, thick] (P6) to[out=-5, in=180] (T9);

\draw[->, draw=purple!80!black, thick] (T1) -- (P10);
\draw[->, draw=purple!80!black, thick] (T2) -- (P10);
\draw[->, draw=purple!80!black, thick] (T3) -- (P10);
\draw[->, draw=purple!80!black, thick] (T4) -- (P11);
\draw[->, draw=purple!80!black, thick] (T5) -- (P11);
\draw[->, draw=purple!80!black, thick] (T6) -- (P11);
\draw[->, draw=purple!80!black, thick] (T7) -- (P12);
\draw[->, draw=purple!80!black, thick] (T8) -- (P12);
\draw[->, draw=purple!80!black, thick] (T9) -- (P12);

\draw[->, draw=orange!80!black, thick] (P10) to[out=0, in=180] (T10);
\draw[->, draw=orange!80!black, thick] (P10) to[out=-5, in=180] (T13);
\draw[->, draw=orange!80!black, thick] (P10) to[out=-10, in=180] (T17);

\draw[->, draw=orange!80!black, thick] (P11) to[out=0, in=180] (T12);
\draw[->, draw=orange!80!black, thick] (P11) to[out=-5, in=180] (T14);

\draw[->, draw=orange!80!black, thick] (P12) to[out=0, in=180] (T11);
\draw[->, draw=orange!80!black, thick] (P12) to[out=-5, in=180] (T15);

\draw[->, draw=green!60!black, thick] (P7) to[out=0, in=180] (T13);
\draw[->, draw=green!60!black, thick] (P8) to[out=5, in=180] (T14);
\draw[->, draw=green!60!black, thick] (P8) to[out=-5, in=180] (T17);
\draw[->, draw=green!60!black, thick] (P9) to[out=0, in=180] (T16);

\draw[->, draw=brown!80!black, thick] (T10) -- (P13);
\draw[->, draw=brown!80!black, thick] (T13) -- (P13);
\draw[->, draw=brown!80!black, thick] (T11) -- (P15);
\draw[->, draw=brown!80!black, thick] (T15) -- (P15);
\draw[->, draw=brown!80!black, thick] (T16) -- (P15);
\draw[->, draw=brown!80!black, thick] (T12) -- (P14);
\draw[->, draw=brown!80!black, thick] (T14) -- (P14);
\draw[->, draw=brown!80!black, thick] (T17) -- (P14);

\begin{scope}[on background layer]
    \node[draw=blue!30, fill=blue!5, rounded corners, thick, inner sep=8pt,
          fit=(P1)(P2)(P3)(P4)(P5)(P6)(P7)(P8)(P9),
          label={[blue!70]above:INPUT VARIABLES}] {};
    \node[draw=red!30, fill=red!5, rounded corners, thick, inner sep=8pt,
          fit=(T1)(T2)(T3)(T4)(T5)(T6)(T7)(T8)(T9),
          label={[red!70]above:RISK ASSESSMENT}] {};
    \node[draw=orange!30, fill=orange!5, rounded corners, thick, inner sep=8pt,
          fit=(P10)(P11)(P12),
          label={[orange!70]above:RISK}] {};
    \node[draw=purple!30, fill=purple!5, rounded corners, thick, inner sep=8pt,
          fit=(T10)(T11)(T12)(T13)(T14)(T15)(T16)(T17) (P13)(P14)(P15),
          label={[purple!70]above:ACTION RULES \& ACTION}] {};
\end{scope}

\node[draw, rounded corners, fill=white, anchor=north west,
      text width=18cm, font=\footnotesize, inner sep=6pt] at (4,-6.5) {
    \textbf{Abbreviations:} A=Autonomy, B=Beneficence, N=Nonmaleficence\\[2pt]
    \textbf{Risk Assessment Rules (R1–R3):}\\
    R1 (T1–T3): (S.low$\land$M.good) $\lor$ (S.med$\land$M.good) $\lor$ (S.low$\land$M.avg) $\to$ Risk.low \quad CF=0.8, \{A\}\\
    R2 (T4–T6): (S.high$\land$M.good) $\lor$ (S.med$\land$M.avg) $\lor$ (S.low$\land$M.bad) $\to$ Risk.med \quad CF=0.7, \{B\}\\
    R3 (T7–T9): (S.high$\land$M.avg) $\lor$ (S.med$\land$M.bad) $\lor$ (S.high$\land$M.bad) $\to$ Risk.high \quad CF=0.9, \{N,B\}\\[2pt]
    \textbf{Action Decision Rules (R4–R10):}\\
    R4 (T10): Risk.low $\to$ Accept \quad CF=0.8, \{A\}\\
    R5 (T11): Risk.high $\to$ TryNow \quad CF=0.9, \{N,B\}\\
    R6 (T12): Risk.med $\to$ TryLater \quad CF=0.7, \{B,A\}\\
    R7 (T13): Risk.low $\land$ LT.low $\to$ Accept \quad CF=0.88, \{A\}\\
    R8 (T14): Risk.med $\land$ LT.med $\to$ TryLater \quad CF=0.78, \{B,A\}\\
    R9 (T15,T16): Risk.high $\lor$ LT.high $\to$ TryNow \quad CF=0.94, \{N,B\}\\
    R10 (T17): Risk.low $\land$ LT.med $\to$ TryLater \quad CF=0.72, \{B\}
};

\end{tikzpicture}
} 
\caption{FPN of the $PatientEDM^+$ Model}
\label{fig:petri_r}
\vspace{1cm}
\end{figure}
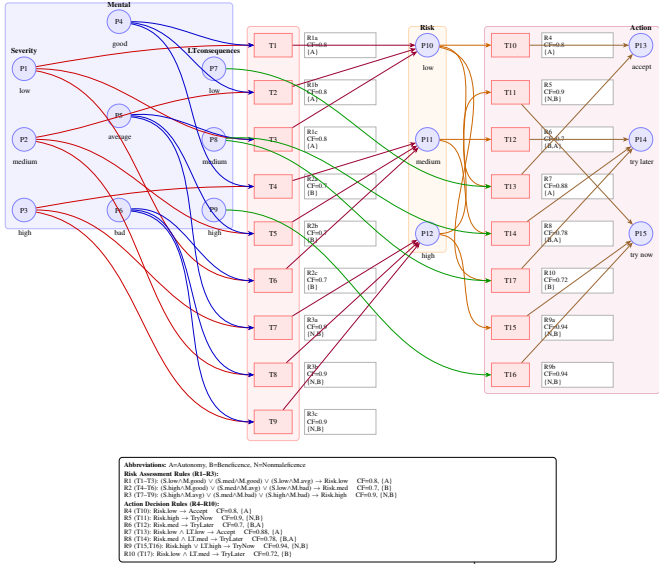

\paragraph{Dynamic Validation (Semantic Correctness):}
This step involves checking for semantic incorrectness through running and reasoning over the FPN. We need to reason the FPN for given inputs and compare the results to their counterparts in the referent models to check if there exists any semantic incorrectness.
The first step in this phase is to map the revised $PatientEDM^+$ model after static validation into FPN. Figure~\ref{fig:petri_r} shows the FPN of the revised model and Listing~\ref{lst6} shows its specifications.

Then, we will calculate action similarities between the $PatientEDM^+$ model output and the expected output by our three referent models. The similarities are calculated using the formulas from Section \ref{vreferent}. Figure \ref{fig:revised_fEDM_output_similarities1} shows the output of the $PatientEDM^+$ model together with the similarities to our three referents. 
Analyzing the current action similarity scores reveals validation failures: the Clinician's $S_A = 0.81$ falls short of their 0.85 threshold despite matching the expected \texttt{tryAgainNow} action, while Patient-Advocate ($S_A = 0.49 < 0.75$) and Hospital-Board ($S_A = 0.49 < 0.80$) reject \texttt{tryAgainNow} in favor of \texttt{tryAgainLater} under 63.1\% crisp risk. To improve these similarities and boost validation rates, we can adjust certainty factors (CFs) in the FERD rules, for instance, increasing R5's CF from 0.9 to 0.95 for stronger \texttt{tryAgainNow} emphasis (high risk), and elevating R6's CF from 0.7 to 0.9 (or adding a new medium-risk rule like D$_x$: (\texttt{risk(medium)} $\wedge$ \texttt{LTconsequences(medium)} $\to$ \texttt{tryAgainLater} (CF: 0.9)) to better align with \texttt{tryAgainLater} expectations. These targeted modifications enhance system-referent coherence while preserving fuzzy semantic tolerances.
\begin{figure}[htb!]
  \centering\includegraphics[width=0.9\linewidth]{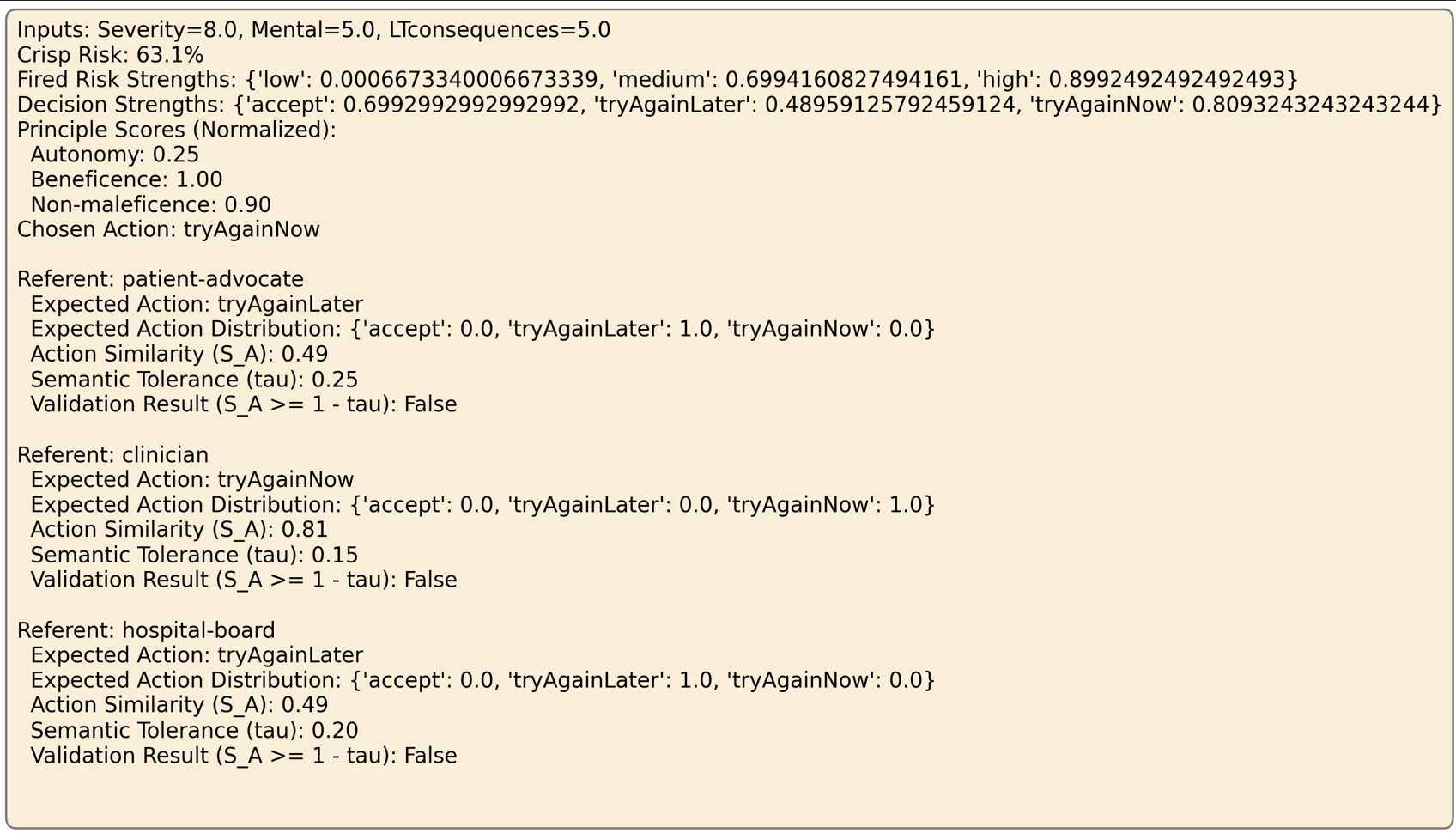}
  \caption{Output of the $PatientEDM^+$ model}
  \label{fig:revised_fEDM_output_similarities1}
   \vspace{1cm}
\end{figure}

Conversely, Figure \ref{fig:revised_fEDM_output_similarities2} shows the output for the same scenario after the explained modification.
After that, we will reason over the FPN applying the reasoning rules and comparing the results to the reasoning references ($RR1$ and $RR2$ in our case) from all three validation referents.

The uncertainty reasoning of the three types of
rules with certainty factors can be summarized as follows \cite{60794}:

\begin{itemize}
	\item Type 1: $R_{i}(\beta_i): P_1 (\alpha_1) \wedge P_2 (\alpha_2)\wedge ... \wedge P_{j-1} (\alpha_{j-1})\rightarrow P_j (\alpha_j)\wedge P_{j+1} (\alpha_{j+1})\wedge ... \wedge P_k (\alpha_k)$ \\
 $\alpha_j = \alpha{j+1} = ... = \alpha_k = min\{\alpha_1, \alpha_2, ... \alpha_{j-1}\}*\beta_i .$
 
	\item Type 2: $R_1(\beta_1) : P_j(\alpha_j) \rightarrow P_1(\alpha_1); \quad
R_2(\beta_2) : P_j(\alpha_j) \rightarrow P_2(\alpha_2); \quad \dots \quad;
R_j(\beta_{j-1}) : P_j(\alpha_j) \rightarrow P_{j-1}(\alpha_{j-1}).$ \\
$\alpha_1 = \alpha_j * \beta_1 . \alpha_2 = \alpha_j * \beta_2 ... \alpha_{j-1} = \alpha_j * \beta_{j-1} .$

	\item Type 3: $R_{i}(\beta_i): (P_1 (\alpha_1)\vee P_2 (\alpha_2)\vee ... \vee P_{j-1} (\alpha_{j-1}))\rightarrow P_j (\alpha_j)$ \\
 $\alpha_j = max \{\alpha_1 , \alpha_2 , ... , \alpha_{j-1} \} * \beta_i .$
\end{itemize}

\begin{lstlisting}[mathescape, caption={Specification of the FPN}, label={lst6},  breaklines=true, breakatwhitespace=true]

FPN = (P, T, D, I, O, $\mu$, $\alpha$, $\beta$);
P = {$p_1, p_2, p_3, p_4, p_5, p_6, p_7, p_8, p_9, p_{10}, p_{11}, p_{12}, p_{13}, p_{14}, p_{15}$};
T = {$t_1, t_2, t_3, t_4, t_5, t_6, t_7, t_8, t_9, t_{10}, t_{11}, t_{12}, t_{13}, t_{14}, t_{15}, t_{16}, t_{17}$};
D = {Severity(low), Severity(medium), Severity(high), Mental(good), Mental(average), Mental(bad), Risk(low), Risk(medium), Risk(high), Action(accept), Action(tryAgainLater), Action(TryAgainNow), LTconsequences(low), LTconsequences(medium), LTconsequences(high)};
$\mu$ = (0.80, 0.70, 0.90, 0.95, 0.70, 0.80);
$p_1$ = Severity(low), $p_2$ = Severity(medium), 
$p_3$ = Severity(high), $p_4$ = Mental(good), 
$p_5$ = Mental(average), $p_6$ = Mental(bad), 
$p_7$ = Risk(low), $p_8$ = Risk(medium), 
$p_9$ = Risk(high), $p_{10}$ = Action (accept),
$p_{11}$= Action(tryAgainLater), 
$p_{12}$ = Action(TryAgainNow), 
$p_{13}$ = LTconsequences(low), 
$p_{14}$ = LTconsequences(medium), 
$p_{15}$ = LTconsequences(high).

\end{lstlisting}

Let us see $RR1\_V1$ in the \textit{Patient\_Advocate} validation referent in Listing \ref{lst3}. Using the FPN to reason, we can obtain the following reasoning results: $\alpha (Risk(high)) = min (0.95,0.9)*0.9 = 0.9*0.9 = 0.81 > 0.80$. therefore, the validation rule of $RR1\_V1$ has passed through the validation.
Let us see $RR2\_V1$ in the validation referent in Listing \ref{lst3}. Using the FPN to reason, we can obtain the following reasoning results: $\alpha (Action(Accept)) = min ( min (0.9,0.8)*0.8, 0.85) = min (0.64, 0.85) = 0.64 < 0.70$. therefore, the validation rule of $RR2$ has not passed through the validation, which means that in the $PatientEDM^+$ model there are some semantic errors that need to be corrected. This error can be corrected/improved by, for example, increasing the \textit{CF} of the rule $R1$ in the $PatientEDM^+$ model from 0.8 to 0.9 so $\alpha (Action(Accept))$ becomes $0.72 > 0.70$ and the model passes the second validation rule of the \textit{Patient\_Advocate} validation referent. Now we repeat the same with the reasoning rules of the other two referent models.

\begin{figure}[htb!]
  \centering\includegraphics[width=0.9\linewidth]{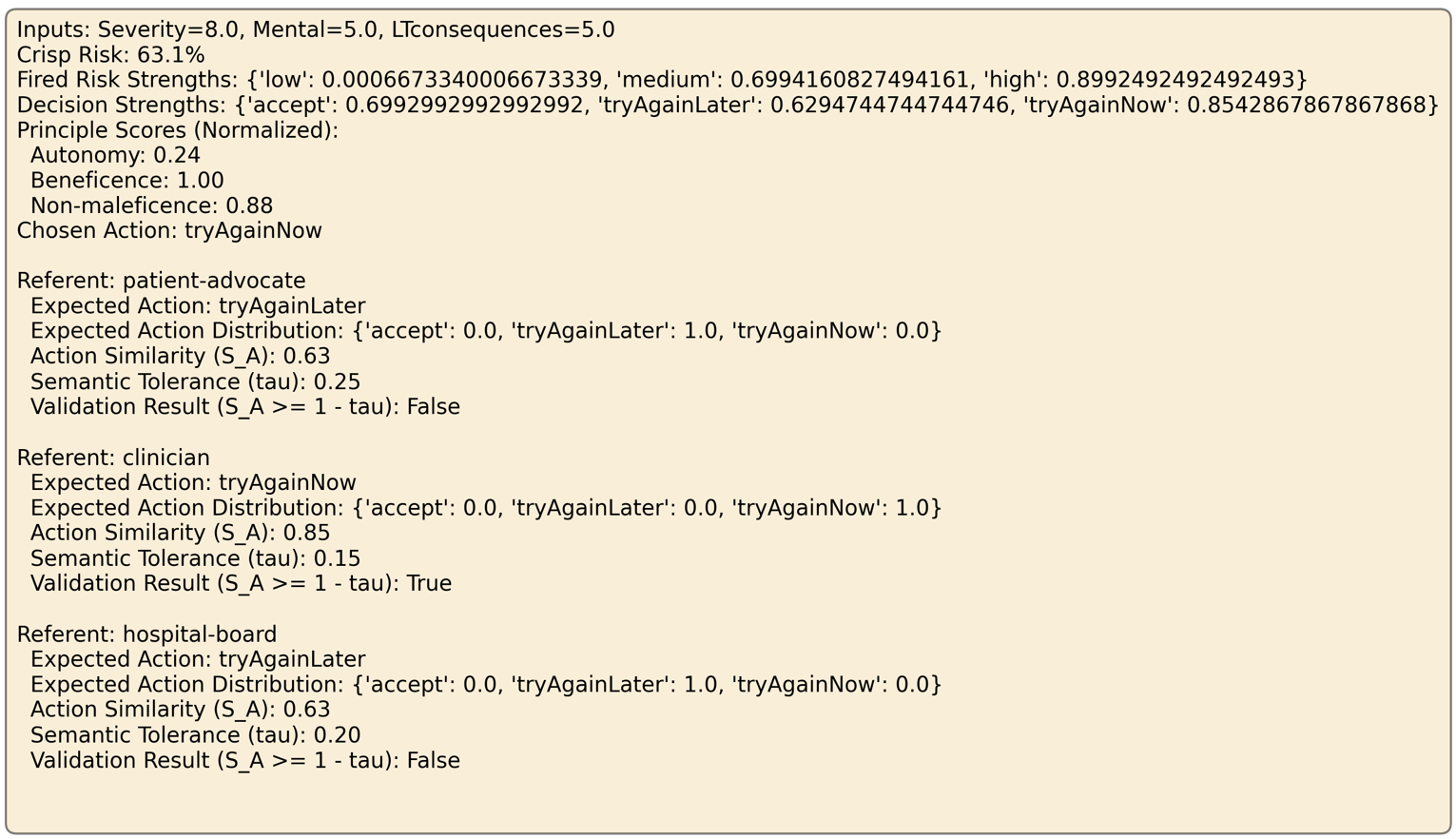}
  \caption{Output of the $PatientEDM^+$ model after modification}
  \label{fig:revised_fEDM_output_similarities2}
  \vspace{1cm}
\end{figure}
\section{Discussion}
\label{dis}



\subsection{Applicability to Reinforcement Learning and LLMs}
While the proposed fEDM framework has been illustrated in a rule-based healthcare case study, an important question concerns its applicability to modern AI systems, particularly those based on reinforcement learning (RL) and large language models (LLMs). These paradigms are increasingly deployed in the safety-critical and human-facing domains, but often raise concerns about ethical reliability and transparency.

\paragraph{Reinforcement Learning (RL)}
RL agents optimize policies with respect to a reward function, but this objective rarely captures the full range of ethical considerations relevant in human–AI interaction. Moreover, RL policies are opaque and can produce unintended behaviors when faced with out-of-distribution states. The \textit{fEDM+} framework can complement RL systems by functioning as an ethical oversight layer as detailed below:

The RL agent proposes an action $a$. This action, together with the current state’s ethically relevant factors (ERFs), is passed through the \textit{fEDM+} model.
If the action produces an ethical risk level above a specified threshold, \textit{fEDM+} either blocks it or suggests alternatives consistent with ethical constraints.
In this way, the RL agent retains adaptive learning capability, while \textit{fEDM+} enforces transparent, verifiable ethical boundaries.

\paragraph{Large Language Models (LLMs)}
LLMs generate context-sensitive responses but are prone to producing ethically problematic or unsafe content. Directly constraining LLM generation is difficult due to the distributed nature of their representations. However, the \textit{fEDM+} framework can serve as a validation layer for LLM outputs: 
Candidate responses are evaluated with respect to fuzzy ethical rules (FERDs) instantiated from domain knowledge.
The \textit{fEDM+} system performs risk assessment (via FERRs) and blocks or flags outputs exceeding acceptable risk thresholds.
This provides a mechanism for embedding domain-specific ethical principles into LLM-based decision pipelines.

Advantages of Oversight Integration. By functioning as an external module, the \textit{fEDM+} framework offers three key advantages for modern AI:
I) Transparency: Ethical constraints are explicitly encoded in fuzzy rules, unlike implicit reward shaping in RL or prompt engineering in LLMs.
II) Verifiability: Rule bases can be formally checked for structural errors via FPN analysis, ensuring that oversight remains sound even as AI models evolve.
III) Flexibility: Expert-defined rules can be supplemented with data-driven rule learning, enabling adaptation to dynamic domains while preserving ethical oversight.

Thus, while \textit{fEDM+} does not replace RL or LLM methods, it provides a practical and formally grounded mechanism for ensuring ethical reliability in their deployment. Future work will explore hybrid implementations where RL and LLM systems are monitored and constrained in real time by \textit{fEDM+} modules.

\subsection{Ethical and Technical Implications}
\label{implic}
\paragraph{Traceability, Trust, and Regulatory Alignment}

The introduction of principle-level traceability through the ETM module has implications that extend beyond interpretability. At a technical level, traceability transforms ethical decision-making from a purely inferential process into an auditable computational pipeline. Each decision is accompanied by structured artifacts: rule activation strengths, certainty factors, and a principle contribution vector. These artifacts enable post hoc inspection, debugging, and systematic review without modifying the underlying inference mechanism.

From a trust perspective, traceability reduces epistemic opacity. Users and domain experts can examine not only what action was recommended, but also how ethical principles influenced the outcome. This supports calibrated trust: stakeholders may accept, contest, or recalibrate decisions based on transparent normative contributions rather than treating the system as a black box. Importantly, traceability allows disagreement to be localized. Instead of rejecting the entire system, a stakeholder can identify that a particular principle weight or rule certainty factor requires adjustment.

In regulatory contexts, traceability aligns with emerging requirements for transparency, accountability, and human oversight in high-risk AI systems. Regulators increasingly demand not only performance guarantees but also evidence of responsible governance. Principle-level logging provides machine-readable justification records that can be archived, audited, and reviewed by oversight bodies. Unlike purely statistical explanation techniques, the ETM module produces semantically meaningful justifications grounded in explicitly encoded ethical principles. This strengthens compliance narratives and supports accountability chains.

Furthermore, traceability enhances human-in-the-loop and human-on-the-loop oversight models. Supervisors can monitor principle contribution profiles over time, detect systematic bias toward particular principles, and intervene where normative drift occurs. The system thus becomes not merely decision-supporting, but governance-supporting: it enables continuous ethical calibration rather than one-time validation.

\paragraph{Pluralism and Explainability in AI Ethics}

The shift from single-referent validation to pluralistic semantic validation has both philosophical and technical significance. Normative pluralism is a persistent feature of real-world ethical environments. Stakeholders differ in their prioritization of principles, interpretation of acceptable risk, and contextual sensitivities. Encoding a single normative referent risks conflating principled disagreement with system malfunction.

By introducing multiple stakeholder referents, the extended framework distinguishes structural validity from normative divergence. A decision may be structurally sound and internally consistent while aligning differently across referents. This reframes explainability: explanations are no longer evaluated against a single normative yardstick but interpreted relative to explicit ethical perspectives.

Technically, pluralistic validation transforms evaluation into robustness analysis across value configurations. The system can quantify degrees of alignment with each referent, identify regions of consensus, and surface zones of principled conflict. This makes the disagreement explicit rather than latent. Philosophically, it operationalizes a procedural conception of legitimacy: ethical acceptability emerges from transparent engagement with multiple normative standpoints rather than implicit endorsement of one.

Explainability, in this context, is not merely descriptive but dialogical. The combination of ETM and pluralistic validation supports structured negotiation between stakeholders and system designers. Adjustments to certainty factors, membership functions, or principle weights can be evaluated systematically against multiple normative baselines. The framework therefore contributes to a model of AI ethics that is computationally explicit, normatively plural, and open to revision.
\subsection{Limitations}
\label{limit}
Despite the extensions introduced in this work, several limitations remain.

First, the definition of fuzzy ethical decision rules and their associated principle annotations continues to rely on expert elicitation. While fuzzy logic provides a structured formalism for representing graded ethical judgments, the initial construction of membership functions, rule bases, and certainty factors requires domain expertise. This introduces potential subjectivity and limits reproducibility across contexts. Although pluralistic validation mitigates the risk of monolithic bias, the quality of the system remains contingent on the adequacy and diversity of expert input.

Second, scalability poses a challenge. As the number of ethically relevant factors, principles, and contextual variables increases, the rule base may grow combinatorially. While FPN-based verification supports structural analysis, large-scale rule systems may become difficult to maintain and computationally intensive to verify. Efficient modularization strategies and automated consistency checking will be necessary for deployment in complex domains.

Third, the current framework is predominantly knowledge-engineered rather than data-driven. While this ensures interpretability and formal control, it limits adaptability to evolving environments. Integrating data-driven learning mechanisms—such as automated extraction of fuzzy rules from empirical datasets—remains an important direction for future work. The challenge lies in combining machine learning with principled traceability without sacrificing verifiability or normative clarity.

Future research will therefore explore hybrid architectures that preserve the formal guarantees of the \textit{fEDM+} framework while enabling scalable rule generation, adaptive calibration, and participatory refinement of principle configurations. Addressing these limitations is essential for transitioning from controlled case studies to deployment in large-scale, real-world AI systems operating under ethical uncertainty.

\section{Conclusions}
\label{con}

\subsection{Summary of contributions}

This paper presented \textit{fEDM+}, an extension of our previous fuzzy Ethical Decision-Making framework (\textit{fEDM})  \cite{DBLP:journals/corr/abs-2507-01410,DyoubandLisi2026}, advancing it from a formally verified fuzzy risk-based architecture toward a principle-aware and pluralism-sensitive model for machine ethics. Building upon our previous work where \textit{fEDM} integrated \textit{fERA}, rule-based ethical reasoning, and FPN–based structural verification validated against a single referent, we introduced two key enhancements: (i) a principle-level Explainability and Traceability Module (\textit{ETM}), and (ii) a pluralistic semantic validation process.

The ETM augments the inference layer with explicit annotations linking ethical decision rules to moral principles and computes a structured \emph{principle contribution vector} for each recommended action. This shifts explainability from rule transparency to normative transparency, enabling principled auditing and calibrated oversight. The pluralistic validation framework replaces the assumption of a single ethical benchmark with a multi-referent approach that encodes diverse stakeholder perspectives, principle priorities, and risk tolerances. Together, these extensions preserve the structural guarantees of the original framework while enhancing interpretability, normative accountability, and contextual robustness.

This work makes three main theoretical contributions. First, it introduces a formal mechanism for aggregating and transparently exposing the influence of ethical principles within fuzzy ethical inference, allowing principle-level reasoning to be inspected and quantified. Second, it proposes a validation paradigm that clearly distinguishes between structural correctness of the reasoning architecture and normative divergence across stakeholder perspectives. Third, it provides a computational operationalization of ethical pluralism within a verifiable reasoning framework, enabling multiple legitimate ethical viewpoints to be formally represented and assessed within a single decision-making system.

Empirically, the healthcare case study demonstrates that \textit{fEDM+} (i) maintains internal consistency under FPN-based verification, (ii) generates interpretable principle contribution profiles, and (iii) systematically differentiates between alignment and principled disagreement across stakeholder referents. The results indicate that decisions remain stable under the extended architecture while becoming significantly more transparent and evaluable.

Overall, \textit{fEDM+} represents a formally (semi-formally) grounded, ethically interpretable, and empirically robust framework for machine ethics. It integrates graded reasoning under uncertainty, structural verification, explicit normative traceability, and pluralism-aware validation within a unified computational model. As such, it contributes toward bridging the gap between formal AI methods and philosophically defensible ethical governance.

\subsection{Future Directions}

Several research directions emerge from this work.

\paragraph{Robustness and Sensitivity Evaluation.}
An important methodological extension is a systematic robustness analysis under variations in membership parameters, rule certainty factors, and input noise. Sensitivity assessment would quantify how ethical outcomes shift under epistemic uncertainty, indicating whether conclusions are stable or parameter‑fragile. Given its methodological depth, this line of work likely merits a dedicated study, as it directly concerns reliability and safety in ethically sensitive AI systems.

\paragraph{Participatory Validation Workshops.}
Pluralistic validation can be strengthened through structured participatory processes. Engaging diverse stakeholders, clinicians, ethicists, engineers, patients, and regulators, in iterative validation workshops would allow referent configurations to be co-designed, contested, and refined. This participatory approach operationalizes democratic legitimacy and ensures that the encoded normative perspectives reflect lived ethical realities rather than abstract assumptions.

\paragraph{Expert-System Cross-Validation with Realistic Scenarios.}
Future work should collect a representative suite of realistic case scenarios and compare \textit{fEDM+} outputs against decisions made by panels of domain experts. In this Expert-System Cross-Validation protocol, expert judgments would serve as provisional 'ground truth', and alignment metrics would quantify the degree to which fEDM replicates human ethical reasoning. Such empirical benchmarking would provide stronger evidence of external validity and practical applicability.

\paragraph{Temporal and Legal Constraints.}
Current \textit{fEDM+} reasoning is primarily static and principle-based. Future extensions may incorporate temporal logic to model evolving obligations, delayed harms, and dynamic consent conditions. Additionally, integrating explicit legal constraints would allow the system to reason over regulatory norms alongside ethical principles, enabling hybrid ethical-legal compliance checking in real-world deployments.

\paragraph{Automatic Fuzzy Rule Learning.}
At present, rule construction and parameter specification rely on expert elicitation. A promising direction involves semi-automated or data-driven extraction of fuzzy rules and membership functions from annotated datasets. The challenge is to preserve interpretability and verifiability while enabling adaptive refinement. Hybrid neuro-symbolic approaches could support learning while maintaining principle traceability.

\paragraph{Integration with Reinforcement Learning and Large Language Models.}
 \textit{fEDM+} can function as an oversight or governance layer for data-driven systems. In reinforcement learning settings, \textit{fEDM+} could provide ethical constraints or reward shaping mechanisms to ensure norm-sensitive policy updates. In LLM-based systems, \textit{fEDM+} could serve as a post-generation evaluator or deliberative controller, assessing candidate outputs against principle contributions and stakeholder referents. Such integration would combine the generative flexibility of modern AI with formally structured ethical oversight.

In conclusion, \textit{fEDM+} moves toward a computational paradigm in which ethical reasoning is not only implementable, but also verifiable, explainable, pluralism-aware and empirically testable. By combining fuzzy logic, formal verification, principle-level traceability, and stakeholder-sensitive validation, the framework contributes to the development of machine ethics systems that are technically rigorous and normatively accountable, which is an essential step for trustworthy AI in high-stakes domains.


\section*{Acknowledgments}
  This work was partially supported by the project FAIR - Future AI Research (PE00000013), under the NRRP MUR program funded by the NextGenerationEU.


\bibliography{fuzzy}


\end{document}